\documentclass[english,a4paper,12pt]{article}

\usepackage{amsxtra, amsfonts, amssymb, amstext, amsmath, mathtools}
\usepackage{hyperref}
\usepackage{amsthm}
\usepackage{booktabs}
\usepackage{nicefrac}
\usepackage{xspace}
\usepackage{url}
\usepackage{graphics,color}
\usepackage[algo2e,ruled,vlined,linesnumbered]{algorithm2e}\SetArgSty{upshape}
\usepackage{algcompatible}
\usepackage{algorithm}
\usepackage{wrapfig}
\usepackage{lmodern}

\usepackage[utf8]{inputenc}
\usepackage{xspace}
\usepackage{amsmath,amsthm,amssymb,mathtools}
\usepackage{lmodern}

\usepackage[algo2e,ruled,vlined,linesnumbered]{algorithm2e}

\usepackage{xcolor}
\usepackage{tikz}
\usepackage{graphicx}

\allowdisplaybreaks[4]
\newtheorem{theorem}{Theorem}
\newtheorem{lemma}[theorem]{Lemma}
\newtheorem{corollary}[theorem]{Corollary}
\newtheorem{definition}[theorem]{Definition}

\newcommand{\oea}{$(1 + 1)$~EA\xspace}

\newcommand{\NSGA}{\mbox{NSGA-II}\xspace}
\newcommand{\NSGAT}{\mbox{NSGA-III}\xspace}
\newcommand{\SMS}{\mbox{SMS-EMOA}\xspace}

\newcommand{\om}{\textsc{OneMax}\xspace}
\newcommand{\jump}{\textsc{Jump}\xspace}
\newcommand{\omm}{\textsc{OneMaxMin}\xspace}
\newcommand{\omim}{\textsc{OneMinMax}\xspace}
\newcommand{\onemax}{\om}
\newcommand{\lo}{\textsc{LeadingOnes}\xspace}
\newcommand{\lotz}{\textsc{LOTZ}\xspace}
\newcommand{\cocz}{\textsc{COCZ}\xspace}
\newcommand{\ojzj}{\textsc{OJZJ}\xspace}
\newcommand{\dlb}{\textsc{DLB}\xspace}
\newcommand{\dltb}{\textsc{DLTB}\xspace}
\newcommand{\wlptno}{\textsc{wlptno}\xspace}

\newcommand{\OO}{\textsc{OM}\xspace}
\newcommand{\LO}{\textsc{LO}\xspace}
\newcommand{\TZ}{\textsc{TZ}\xspace}

\newcommand{\LFC}{\textsc{LFC}\xspace}
\newcommand{\RR}{\textsc{RR}\xspace}
\newcommand{\uRR}{\textsc{uRR}\xspace}
\newcommand{\otzt}{\textsc{OneTrapZeroTrap}\xspace}
\newcommand{\trap}{\textsc{Trap}\xspace}

\newcommand{\R}{\ensuremath{\mathbb{R}}}

\newcommand{\N}{\ensuremath{\mathbb{N}}} 
\newcommand{\Z}{\ensuremath{\mathbb{Z}}}

\DeclareMathOperator{\cDis}{cDis}
\DeclareMathOperator{\SurSel}{SurvivalSelection}

\DeclareMathOperator{\HV}{HV}

\newcommand{\eps}{\varepsilon}

\begin{document}
\sloppy
\date{}
\title{Proven Advantage of Multiobjective Evolutionary Algorithms for Problems with Different Degrees of Conflict}
\author{Weijie Zheng\thanks{School of Computer Science and Technology, International Research Institute for Artificial Intelligence, Harbin Institute of Technology, Shenzhen, China}}
\maketitle

\begin{abstract}
The field of multiobjective evolutionary algorithms (MOEAs) often emphasizes its popularity for optimization problems with conflicting objectives. However, it is still theoretically unknown how MOEAs perform compared with typical approaches outside this field.

This paper conducts such a systematic theoretical comparison on problem classes with different degrees of conflict. With $\omm_k$ depicting $k\in[0..n]$ degrees of conflict, we show the difficulties of two typical non-MOEA approaches,  {the} scalarization (weighted-sum) and  {the} $\eps-$constraint approach. We prove that for any set of weights, the set of optima formed by  {the} scalarization approach cannot cover its full Pareto front for $k>2$. Although constrained problems constructed from $\eps-$constraint approach ensure the full coverage, general ways (via exterior or nonparameter penalty functions) to solve these constrained problems encounter difficulties. The nonparameter penalty function way cannot guarantee the full coverage, and the exterior way covers the Pareto front with expected $O(\max\{k,1\}n\ln n)$ number of function evaluations, but only with careful settings of $\eps$ and $r$ ($r>1/(\eps+1-\lceil \eps \rceil)$).

In contrast, MOEAs efficiently solve $\omm_k$ without careful designs. We prove the same expected runtime of $O(\max\{k,1\}n\ln n)$ for the (G)SEMO, MOEA/D, \NSGA, and \SMS. 

Our brief discussions on a bi-objective \lo variant with different degrees of conflict show similar findings.
\end{abstract}

\section{Introduction}
Many real-world applications have multiple objectives to optimize simultaneously, and multiobjective evolutionary algorithms (\emph{MOEAs}) are widely utilized to solve them. The theory community  {mainly} uses the mathematical runtime as a tool to evaluate the performance of MOEAs rigorously, and to understand the working principles better. The runtime analysis of MOEAs starts from seminal papers by Laumanns et al.~\cite{LaumannsTZWD02,LaumannsTZ04} and has made the breakthrough from toy algorithms (GSEMO) to practical algorithms like the \NSGA~\cite{ZhengLD22,ZhengD23aij} and its non-dominated sorting variants (\NSGAT~\cite{WiethegerD23,OprisDNS24}, \SMS~\cite{BianZLQ25}) as well as the MOEA/D~\cite{LiZZZ16}. However, there are two fundamental questions that are still not well answered from the theoretical side.  

As is known, MOEAs are not the only way to solve multiobjective optimization problems. Typical alternatives include scalarization (weighted-sum) and the $\eps$-constraint method.
The scalarization method combines all objectives
into an integrated one via assigning weights to each objective, and then solves this single-objective optimization problem. The $\eps$-constraint method selects one objective to optimize and converts
remaining objectives into inequality constraints (with constants $\eps_1,\eps_2,\dots$ as corresponding lower bounds for maximization problems), and then solves this constrained problem. However, the existing results for multiobjective optimization in our theory community mainly focus on the runtime of the MOEAs. The systematic theoretical comparison between MOEAs and typical non-MOEA ways is still not clear.

The field of MOEAs often emphasizes its popularity for optimization problems with conflicting objectives. The objectives can have different degrees of conflict when the decision variables take different values. For example, delivery personnel are required to transport spill-prone liquid foods, like coffee, soups, and congee, to customers rapidly and without leakage or quality deterioration. On high-quality road segments, delivery-time minimization and quality deterioration minimization are mutually non-conflicting objectives. On poor road segments, the two objectives conflict: detours maintain absence of spillage but lengthen travel time, whereas direct routes shorten time at the cost of increased spillage. There is a natural question about the relationship between the performance of the MOEAs and the conflict degree of the objectives, compared with other non-MOEA ways discussed before. Even a basic question exists about how MOEAs behave for problems with no conflict, compared to single-objective optimization  {algorithms}.

This paper takes such a step to tackle the above two questions. We resort to the well-analyzed \omim and \cocz benchmarks where the  {former} considers the extreme conflict between two objectives, and  {the} two objectives  {of} the  {latter} conflict in half  {of the} bit positions. We consider a generalized class called $\omm_k$, where $k$ is the number of bits in which two objectives conflict, and we call it the degree of conflict. This benchmark class includes \omim~for $k=n$ ($n$ the problem size), $k=n/2$ for \cocz, and $k=0$ for no conflict. 
We note that a similar function was independently proposed in~\cite{AntipovKR24}, and we will discuss its details in Sections~\ref{ssec:bench} and~\ref{sec:ommk}.

For this class of problems, we will show that two typical non-MOEA approaches, scalarization (weighted-sum) and $\eps$-constraint approach, are difficult or inconvenient to use. We prove that for any set of weights, the set of optima (one optimum corresponds to one constructed single-objective problem) found by the scalarization approach cannot cover the full Pareto front of $\omm_k$ with $k>2$ (Theorem~\ref{thm:ws}). Although the optima set of constrained problems constructed via $\eps$-constraint approach can cover the full Pareto front, the commonly used ways (via exterior or nonparameter penalty functions) to solve such constrained problems encounter difficulties. The penalty function approach transforms the constrained problem into an unconstrained one. The nonparameter penalty function approach
adds the unscaled violation to the objective when a constraint is violated, leaving it unchanged otherwise, and the exterior penalty function performs the same augmentation, but first scales the violation by a positive penalty coefficient.
We will prove that the nonparameter penalty function way cannot guarantee the full coverage, and the exterior way helps (with expected runtime of $O(n\ln n)$ for the randomized local search algorithm for reaching any Pareto front point, see Theorem~\ref{thm:rls}) but with careful settings of $\eps$ and $r$ ($r>1/(\eps+1-\lceil \eps \rceil)$) (Corollaries~\ref{cor:r01} and~\ref{cor:cor}). Specifically, since the Pareto front size is $k+1$, $k+1$ different settings of $\eps$ are needed for the exterior way to achieve full coverage of the Pareto front, which results in a $O(\max\{k,1\} n\ln n)$ runtime for the randomized local search algorithm.

In contrast to the careful designs  {required in} in the above-mentioned non-MOEA approaches, we will prove that in the typical settings, the popular toy MOEAs (the (G)SEMO), and the practical ones (the MOEA/D, the \NSGA, and the \SMS) cover the full Pareto front of $\omm_k$ in expected $O(\max\{k,1\}n\ln n)$ function evaluations (Theorems~\ref{thm:moead} and~\ref{thm:moeas}), same as the the exterior way. In particular, for solving the problem with two identical objectives ($k=0$), the MOEAs will achieve the same runtime as the counterpart RLS or \oea for single-objective optimization. 

These results for the MOEAs hold for the following natural parameter settings. The result for the (G)SEMO does not have specific settings. The ones for the practical ones hold when the population size used in the \NSGA and the \SMS is $\Theta(k+1)$, the reference point used in \SMS is dominated by all Pareto optimal solutions, and the number of subproblems in MOEA/D is set to be the size of the Pareto front. Increasing population size is a generally used strategy (although the recent theoretical result~\cite{ZhengD24tec} shows its inefficiency for more than two objectives on \omim variants), and a good runtime ($O(Nn\log n)$ for the \NSGA, and $O(\mu n\log n)$ for the \SMS where $N$ and $\mu$ are population sizes respectively) is still reached for a reasonably large population size. The restriction on the reference point is not that difficult to fulfill if one chooses a vector with small enough components for the maximization problem. Although the requirement on the number of subproblems in the MOEA/D seems strict, the very recent work~\cite{DoerrKW24} shows that starting from the coverage of a subset of the Pareto front, the MOEA/D still covers the full Pareto front in a reasonable runtime for the \omim problem. Now we compare them with parameters used in the efficient exterior way to solve the constrained problems constructed from the $\eps$-constraint approach, which also achieves an efficient performance. With respect to $k$ degrees of conflict, this method needs to construct a specific set of $\eps$ parameters (with size at least $k+1$) in a good distribution, while the \NSGA and the \SMS only require their population size at least $4(k+1)$ and $k+1$ respectively, and the MOEA/D requires the number of subproblems being $k+1$. Besides, this method also needs another positive penalty parameter $r$ larger than $1/(\eps+1-\lceil \eps \rceil)$.

To see the generality of the above findings, we will have brief discussions on a bi-objective \lo variant with different degrees of conflict, called $\lotz_k$ with $k$ for the conflict degree. The difficulty of the scalarization and the inconvenience of the $\eps$-constraint method are observed as well. The MOEAs ((G)SEMO, \NSGA, and \SMS) will be proved to have the same runtime of $O(\max\{k,1\}n^2)$ as the $\eps$-constraint method solved by the exterior penalty way with proper coefficients.

The remainder of this paper is organized as follows. Section~\ref{sec:pre} introduces the preliminaries of bi-objective pseudo-Boolean optimization and a brief review of the bi-objective benchmarks that are commonly used in the theory community. The $\omm_k$ benchmark class is in Section~\ref{sec:ommk} to depict the different degrees of conflict between two objectives. Sections~\ref{sec:typ} and~\ref{sec:eps} show the difficulties or inconvenience of two typical non-MOEA approaches, and the theoretical efficiency of commonly analyzed MOEAs is discussed in Section~\ref{sec:moeas}. The performance comparison on a bi-objective \lo variant is briefly discussed in Section~\ref{sec:lotz}. Section~\ref{sec:con} concludes this paper.

\section{Preliminaries}\label{sec:pre}
Following the standard notation, for $a,b\in \Z$ and $a\le b$, we write $[a..b]$ to denote $[a,b] \cap \Z$. For $x=(x_1,\dots,x_n)\in\{0,1\}^n$, we use $x_{[a..b]}$ to denote $(x_a,\dots,x_b)$, use $|x|_1$ to denote the number of ones in $x$, and abbreviate a sub-bitstring $(1,\dots,1)$ with the length of $k$ by $1^k$.  {Besides, for any two vectors $u=(u_1,\dots,u_\ell),v=(v_1,\dots,v_\ell)\in\R^{\ell}$ with $\ell\in \N$, we say $u>(\ge,<,\le,=)~v$ if $u_i>(\ge,<,\le,=)~v_i$ for all $i\in[1..\ell]$.}

\subsection{Bi-objective Pseudo-Boolean Optimization}\label{ssec:bpo}
This work only considers bi-objective pseudo-Boolean optimization (maximization). That is, we consider to maximize $f=(f_1,f_2): \{0,1\}^n\rightarrow \R^2$. Different from single-objective optimization, not all solutions in multiobjective optimization are comparable. We say $x$ \emph{weakly dominates} $y$, denoted by $x\succeq y$, if and only if $f_i(x) \ge f_i(y)$ for all $i\in\{1,2\}$. We say $x$ \emph{dominates} $y$, denoted by $x\succ y$, if and only if $f_i(x) \ge f_i(y)$ for all $i\in\{1,2\}$ and at least one of the inequalities is strict. We say $x$ and $y$ incomparable if and only if neither $x\succeq y$ nor $y\succeq x$. If $x$ cannot be dominated by any other solutions, then we say $x$ is \emph{Pareto optimal}, and all Pareto optimal solutions form the \emph{Pareto set}. The set of the function values of the Pareto set is called the \emph{Pareto front}, and we call a vector on the Pareto front \emph{a Pareto front point}. The \emph{runtime} commonly discussed in the evolutionary theory community is the number of iterations or function evaluations to reach a predefined goal, like the full Pareto front coverage, that is, to reach a set of solutions whose function values contain all Pareto front points.

\subsection{Literature Review on Bi-objective Benchmarks Used For Theoretical Analysis}\label{ssec:bench}
The popular bi-objective benchmarks are the counterparts of the well-analyzed single-objective problems. Based on the classic single-objective \onemax benchmark that counts the number of ones in the bit string, the bi-objective \cocz~\cite{LaumannsTZWD02} has \onemax as the first objective, and the second objective aligns with the first in the first half of the
bit string but conflicts with it in the second half.
Another counterpart \omim~\cite{GielL10} has \onemax as one objective and the other objective that conflicts in all bits (that is, this objective is \onemax w.r.t. $\bar{x}=1-x$). For the classic single-objective \lo benchmark, the bi-objective \lotz~\cite{LaumannsTZ04} contains \lo as one objective and the other objective of a variant of \lo w.r.t. $\bar{x}=1-x$ but calculates from right to left. \wlptno~\cite{QianYZ13} generalizes \lotz~by shifting the search space from $\{0,1\}^n$ to $\{-1,1\}^n$ and simultaneously maximizing the leading positive ones and the trailing negative ones with weights. The popular multimodal \jump class was used to be one objective in \ojzj~\cite{ZhengD23ecj}, and \ojzj has the other objective of \jump w.r.t. $\bar{x}=1-x$. \dltb~\cite{ZhengLDD24} is the bi-objective counterpart of multimodal \dlb (a deceptive version of \lo where the critical block (two positions) prefers the deceptive value of $00$ instead of $11$ before the optimal $1^n$ is reached) and is constructed in a similar way as \lotz. These bi-objective benchmarks are well and intensively studied in the theory community~\cite{Giel03,BrockhoffFN08,DoerrKV13,DoerrGN16,BianQT18ijcaigeneral,HuangZCH19,HuangZ20,OsunaGNS20,ZhengLD22,BianQ22,ZhengD23aij,ZhengD25approx}. Some artificial benchmarks are also constructed for specific aims. \RR~and \uRR~\cite{DangOS24} inspired by the single-objective royal road functions are designed to demonstrate the possible exponential speed-up of employing the crossover operator. \otzt~\cite{DangOS24gecco} is the bi-objective counterpart of the \trap~benchmark, with the first objective of the original trap and the other objective of \trap~w.r.t. $\bar{x}=1-x$.  

From the construction of these benchmarks, we know that the conflict of two objectives exists and usually in an extreme manner. 
Very recently, independent of our arXiv version of this work~\cite{Zheng24}, Antipov, K\"otzing, and Radhakrishnan~\cite{AntipovKR24} generalized the \cocz and \omim to the minimization of $(\OO_a,\OO_b)$ where $a$ and $b$ are two optimal bitstrings w.r.t. two objectives respectively, and $\OO_s(x)$ calculates the Hamming distance between $x$ and $s$ for $s\in\{a,b\}$. Here the Hamming distance between $a$ and $b$ can be regarded as the degree of conflicts. 
They also proved the expected runtime of $O(kn\ln n)$ for (G)SEMO, which partially answered our second question on the relationship between the performance of the (G)SEMO and the conflict degrees of objectives. However, our main question about the systematic comparison between typical non-MOEAs and MOEAs (including practical ones) is still not answered. Besides, they also considered a linear function class denoted by \LFC where an offset is added and each bit is accompanied with a weight, and called the bit conflicted for two \LFC with different weights if two corresponding weights are non-zero and have different signs. For the bi-objective problems constructed with \LFC, no relevant theoretical results are given. 


\section{$\omm_k$}\label{sec:ommk}
In this section, we will introduce the bi-objective benchmark class depicting the different degrees of the conflict between two objectives.

As discussed in Section~\ref{ssec:bench}, {independent of ours,} Antipov, K\"otzing, and Radhakrishnan~\cite{AntipovKR24} considered different degrees of conflict for two objectives, generalized from \cocz and \omim. Formally, they consider the minimization of 
\begin{align}
    (\OO_a,\OO_b),
    \label{eq:omc}
\end{align}
where $a,b\in\{0,1\}^n$, and the function $\OO_d: \{0,1\}^n \rightarrow \R$ is defined by $\OO_d(x)={H(d,x)}$, that is, the Hamming distance between two bitstrings $d$ and $x$. For this problem, $k=H(a,b)$ can be regarded as the degree of conflict. We omit the introduction about a more general conflict defined on the generalized linear function class \LFC, as no theoretical results are given for this kind and our results will focus on more intuitive and easy problems but conjecture our findings will be helpful for the future analysis.

To better fit into the structures of \cocz and \omim, we now introduce the essentially identical but more intuitive benchmark class, comparing to (\ref{eq:omc}). We note that the obtained results in this work can be easily applied to the minimization of (\ref{eq:omc}). 
Following the similar structure in \cocz, the following benchmark class makes the first and the second objectives concur in the first $n-k$ bit positions and conflict in the last $k$ bit positions. $k$ is the degree of conflict between two objectives. See Definition~\ref{def:ommk}.
\begin{definition}
The $\omm_k$ function $\{0,1\}^n \rightarrow \R^2$ is defined by
\begin{equation}
\begin{split}
    f^k(x)&{}={}(f^k_1(x),f^k_2(x))
={}\left(\sum_{j=1}^n x_j, \left(\sum_{j=1}^{n-k} x_j\right) + \left(\sum_{j=n-k+1}^n 1- x_j \right)\right)
\end{split}
\label{eq:coczk}
\end{equation}
for $x=(x_1,\dots,x_n)\in\{0,1\}^n$.
\label{def:ommk}
\end{definition}
We easily see that it includes \cocz for $k=n/2$, and \omim (swapping the orders of two objectives) for $k=n$. 
{The following lemma shows the Pareto front for $\omm_k$.}
{
\begin{lemma}
Let $M$ denote the Pareto front of $\omm_k$. Then
\begin{align*}
    M=\{(n-k+i,n-i) \mid i\in[0..k]\}.
\end{align*}
\label{lem:PF}
\end{lemma}
}
\begin{proof}
{We first show that $y\not\succ x$ for any $x$ with $f^k(x)\in M$ and for any $y\in\{0,1\}^n$. If $y\succ x$, from the definition of dominance, we know 
\begin{align}
f^k_1(y)+f^k_2(y) > f^k_1(x)+f^k_2(x).
\label{eq:yx1}
\end{align}
From (\ref{eq:coczk}) we have
\begin{align*}
    f^k_1(y)+f^k_2(y)=2|y_{[1..n-k]}|_1+k\le 2(n-k)+k=2n-k.
\end{align*} 
Noting that $f^k_1(x)+f^k_2(x)=2n-k$ from the definition of $M$, we know
\begin{align*}
f^k_1(y)+f^k_2(y) \le f^k_1(x)+f^k_2(x),
\end{align*}
which contradicts (\ref{eq:yx1}). Then $y\not\succ x$.}

{Now it remains to prove that for any $y\in \{0,1\}^n$ with $f^k(y)\notin M$, there exists $z\in\{0,1\}^n$ such that $z\succ y$.
From $f^k(y)\notin M$, we easily  {see} $y_{[1..n-k]}\neq 1^{n-k}$. Let $z$ with $z_{[1..n-k]}=1^{n-k}$ and $z_{[n-k+1..n]}=y_{[n-k+1..n]}$. Then $f_1^k(z) > f_1^k(y)$ and $f_2^k(z)>f_2^k(y)$. Thus $z\succ y$.} 
\end{proof}
In this work, we will discuss the runtime (number of function evaluations or iterations) that the algorithms need to cover the full Pareto front.

The following lemma calculates the maximal number of incomparable solutions, which will be used in the later sections.
\begin{lemma}
The maximal number of pairwise non-dominated function values w.r.t. $\omm_k$ is $k+1$.
\label{lem:popsize}
\end{lemma}
\begin{proof}
{Let $V$ be the set of mutually incomparable solutions. We will show that for any incomparable $x,y\in V$, $|x_{[n-k+1..n]}|_1\ne|y_{[n-k+1..n]}|_1$. If $|x_{[n-k+1..n]}|_1=|y_{[n-k+1..n]}|_1$, noting that \begin{equation*}
\begin{cases}
    f^k(x)=(|x_{[1..n-k]}|_1+|x_{[n-k+1..n]}|_1,|x_{[1..n-k]}|_1+k-|x_{[n-k+1..n]}|_1)\\
    f^k(y)=(|y_{[1..n-k]}|_1+|y_{[n-k+1..n]}|_1,|y_{[1..n-k]}|_1+k-|y_{[n-k+1..n]}|_1),
\end{cases}
\end{equation*}
we have that
\begin{itemize}
    \item if $|x_{[1..n-k]}|_1\ge|y_{[1..n-k]}|_1$, then $f^k(x)\ge f^k(y)$, thus $x \succeq y$;
    \item if $|x_{[1..n-k]}|_1<|y_{[1..n-k]}|_1$, then $f^k(x)<f^k(y)$, thus $y \succ x$,
\end{itemize}
which contradicts that $x$ and $y$ are incomparable. Since $|x_{[n-k+1..n]}|_1\in[0..k]$, we have $|V|\le k+1$. }
With the Pareto front size of $k+1$ {from Lemma~\ref{lem:PF}}, this lemma is proved.
\end{proof}

\section{Difficulty of the Scalarization Approach}\label{sec:typ}
In this section and the next section, we will discuss how the typical non-MOEA approaches (scalarization and $\eps$-constraints) behave to different degrees of the objectives, that is, how they solve the $\omm_k$ defined in Section~\ref{sec:ommk}. We will see the difficulties or the inconveniences of these two approaches.

\subsection{Scalarization (Weighted-Sum)}
One well-known approach to tackle multiobjective optimization is to use scalarization (or called weighted-sum approach) to combine all objectives into an integrated one via assigning weights to each~\cite{CaramiaD20}. Formally, the scalarization will reformulate the maximization of the original bi-objective $\omm_k$ (Definition~\ref{def:ommk}) to the maximization of the following function.
\begin{definition}
Let $w\in\R$. The scalarization function $f_w^k:\{0,1\}^n\rightarrow \R$ of $\omm_k$ is defined by
   \begin{equation}
   \begin{split}
&{}f^k_w(x)={}wf^k_1(x)+(1-w)f^k_2(x)\\
&{}={}\left(\sum_{j=1}^{n-k} x_j \right) +k(1-w)+(2w-1)\left(\sum_{j=n-k+1}^n x_j \right)
\end{split}
\label{eq:ws}
\end{equation} 
for $x=(x_1,\dots,x_n)\in\{0,1\}^n$.
\end{definition}
Obviously, the optimum is $1^n$ for $2w-1>0$, and $1^{n-k}0^{k}$ for $2w-1<0$. For $2w-1=0$, the optima are $1^{n-k}*$ with any $*\in\{0,1\}^{k}$.

\subsection{Solving the Constructed Problems}
Usually, an algorithm for one single-objective problem only provides one solution. Here we assume that the algorithm will be terminated when one Pareto optimum is reached. Then one problem contributes to at most one Pareto front point (it depends on whether this Pareto front point has already been reached or not). Hence, to cover the full Pareto front, we need to carefully select a set of weights $w$ to construct the minimal number of problems (and also need to carefully select corresponding algorithms).\footnote{{Note here that a sequence of solutions, say $x_{(w)}^1,\dots,x_{(w)}^{T_w}$, will be generated in the optimization process of each single-objective problem w.r.t. $w$. For a set of weights $\{w_1,\dots,w_K\}$ for some $K\in \N$, this work only consider whether the set of final returned solutions $\left\{x_{(w_1)}^{T_{w_1}},\dots,x_{(w_K)}^{T_{w_K}}\right\}$, instead of the historically generated ones $\left\{x_{(w_1)}^1,\dots,x_{(w_1)}^{T_{w_1}},\dots,x_{(w_K)}^1,\dots,x_{(w_K)}^{T_{w_K}}\right\}$, covers the full Pareto front.} }
However, the suitable choices of weights are not known beforehand. We point out a worse situation in the following theorem that even using an infinite number of weights and different algorithms, the weighted-sum approach cannot fully cover the Pareto front of $\omm_k$. {We note that from the proof of the following theorem, we see a positive chance for a coverage of the Pareto front when restarting $w=1/2$ for many times. However, since one cannot easily have the prior knowledge of repeating a specific weight many times (and how many times should we repeat is not well answered as well), we will not discuss the possibility of this impractical case in this work.} 
\begin{theorem}
Let {$M$ be defined as in Lemma~\ref{lem:PF} and} $k>2$. Let {$S \subset \R$ be a set of $w$, and $x_w$ be one global optimum of $f^k_w$.}
{Let $F=\{f^k(x_{w})\mid w\in S\}$. Then $M \not\subset F$}.
\label{thm:ws}
\end{theorem}
\begin{proof}
{If there is $w\in S$ with $2w-1>0$, then as discussed before, $x_w=1^n$ and thus $(n,n-k)\in F$. Similarly, if there is $w\in S$ with $2w-1<0$, then $x_w=1^{n-k}0^k$ and thus $(n-k,n)\in F$. If $w=1/2\in S$, then $x_{1/2}\in\{1^{n-k}*\mid *\in\{0,1\}^k\}$. Since only one optimal solution is picked for $f^k_{1/2}$, we know that $x_{1/2}$ will contribute at most one more Pareto front point into $F$, which depends on whether $x_{1/2}\in\{1^n,1^{n-k}0^k\}$ or not.}
Hence, for any $S$, maximizing $f^k_{w\in S}$ reaches at most three Pareto front points, which is less than {$|M|=k+1$}. 
\end{proof}

We note that the drawback of the weighted-sum approach has been already recognized for a long time. For example, in 1997, Das and Dennis~\cite{DasD97} gave the geometrical explanation of two major drawbacks of this approach. The first drawback is that no weight will lead to solutions in the non-convex part of the non-convex Pareto curve. The second drawback is that the even spread of weights does not result in an even spread on the Pareto curve.  {In textbooks on} multiobjective optimizations~\cite{Deb01,Coello07}, the major drawback  {is} regarded as ``this approach cannot generate concave portions of the Pareto front regardless of the weights used''~\cite{Coello07}.~\cite[Section~3.1]{Deb01} also listed other practical disadvantages, like the second drawback mentioned in~\cite{DasD97}. Besides, with no example given, it was pointed out that multiple solutions w.r.t. a specific weight may result in a {proper subset of} Pareto set. Since a {set with some but not all Pareto optima} does not always mean {an incomplete} coverage of the Pareto front, it is a straightforward question  {whether} all points in the convex part of the Pareto front curve can be covered with a set of mutually different weights ({which} can be uncountable). In fact, our Theorem~\ref{thm:ws} indicates a negative answer. 

Note that the discrete solution space of $\{0,1\}^n$ results in a non-convex domain of definition. In order to construct a convex Pareto front curve, we just need to change the domain of definition to $[0,1]^n$. In this case, the Pareto front curve is a line segment of $\{(n-k+i,n-i)\mid i\in[0,k]\}$, which is obviously a convex set (from the definition of the convex set). With this domain of definition, the optimum for maximizing (\ref{eq:ws}) is also $1^n$ for $2w-1>0$, and $1^{n-k}0^{k}$ for $2w-1<0$. For $2w-1=0$, the optima are $1^{n-k}*$ with any {$*\in[0,1]^{k}$}. Therefore, with the same proof, we note that Theorem~\ref{thm:ws} also holds for this domain of definition. That is, at most three Pareto front points will be reached for any set of weights. Noting the uncountable Pareto front points ($\{(n-k+i,n-i)\mid i\in[0,k]\}$), we proved our claim that we cannot cover the convex part of the Pareto front cover by setting a set of mutually different weights (conditional on that one weight provides one solution).

\section{Inconveniences of $\eps$-Constraint Approach}\label{sec:eps}
The previous section shows the difficulty of using the scalarization approach to optimize the bi-objective $\omm_k$. In this section, we will discuss the inconveniences of another well-known way of the $\eps$-constraint approach.

\subsection{$\eps$-Constraint Approach}\label{ssec:epsc}
The $\eps$-constraint approach is also a well-known way to tackle the multiobjective optimization~\cite{Deb01,CaramiaD20}. For the bi-objective maximization problem, the $\eps$-constraint approach selects one objective to optimize and requires the other objective not smaller than $\eps$ as a constraint. Formally, this approach reformulates (\ref{eq:coczk}) into the following problem (w.l.o.g., treating $f^k_2$ as a constraint).
\begin{equation}
\begin{split}
&\max f^k_1 (x)=\sum_{j=1}^n x_j\\
&\text{~s.t.~} f^k_2(x) = \left(\sum_{j=1}^{n-k} x_j\right) + \left(\sum_{j=n-k+1}^n 1-x_j\right)\ge \eps
\end{split}
\label{eq:con}
\end{equation}
For each $\eps$, solving (\ref{eq:con}) by an algorithm provides one solution. Recall that our goal for (\ref{eq:coczk}) is to obtain a diverse set of solutions whose function values cover the whole Pareto front of $f^k$. Hence, with {the} $\eps$-constraint approach, we need to choose a minimal number (at least $k+1$) of $\eps$ instances to meet our goal.

We now discuss the optimal solutions for different instances of $\eps$. Since $\max f_2^k=n$, we know that no feasible solution exists for $\eps > n$. If $\eps\le n-k$, then the optimal solution is $1^n$. If $\eps \in (n-k,n]$, 
let $\eta=f_2^k(x)=|x_{[1..n-k]}|_1+k-|x_{[n-k+1..n]}|_1$, then
\begin{align*}
f^k_1(x)&{}={}|x_{[1..n-k]}|{_1}+|x_{[n-k+1..n]}|{_1}\\
&{}={}|x_{[1..n-k]}|{_1}+(|x_{[1..n-k]}|{_1}+k-\eta)\\
&{}={}2|x_{[1..n-k]}|{_1}+k-\eta.
\end{align*}
Hence, the maximal $f^k_1$ value of {$2n-k-\lceil \eps \rceil$ among all feasible solutions} is reached when $|x_{[1..n-k]}|_1=n-k$ and $\eta=\lceil\eps\rceil$, that is, when $x_{[1..n-k]}=1^{n-k}$ and $|x_{[n-k+1..n]}|_1=n-\lceil \eps \rceil$. We note that not all solutions with $f^k_1$ value of {$2n-k-\lceil \eps \rceil$} are feasible. Together with the Pareto front of $\omm_k$ in Lemma~(\ref{lem:PF}), we know that for a set of constrained problems (\ref{eq:con}) with {$\eps=\eps_i, i\in[0..k]$ where} ${\eps_i\in(n-k+i-1,n-k+i],i\in[1..k]}$ and ${\eps_0\in(-\infty,n-k]}$, the set of optimal solutions results in a full coverage for $\omm_k$.

\subsection{Solving the Constrained Problem}
Now we consider solving the constrained optimization problem (\ref{eq:con}). For evolutionary algorithms, the most used way is based on the penalty functions~\cite{ZhouH07}. With penalty functions, the constrained problem is reformed as an unconstrained problem to solve. Following~\cite{ZhouH07}, we discuss the following two kinds of penalty functions: an exterior penalty function and a nonparameter penalty function. 

\subsubsection{Exterior Penalty Function}\label{sssec:expenalty}
To penalize the constraint when it is violated, a simple exterior penalty function with penalty coefficient $r>0$ for (\ref{eq:con}) can be constructed as maximizing 
\begin{equation}
g(x)=f^k_1(x)+r\min\{0,f^k_2(x)-\eps\}.
\label{eq:epf}
\end{equation}
If the solution $x$ is infeasible (that is, $f_2^k(x) < \eps$), then (\ref{eq:epf}) becomes $f_1^k(x)-r(\eps-f_2^k(x))$, which penalizes $f_1^k(x)$ by the absolute difference with a factor of $r$. If $x$ is feasible, then $g(x)=f_1^k(x)$. 
\\
\\
{A. Optimal Solution Set}

{It is not difficult to see that } the optimal solution set for maximizing (\ref{eq:epf}) {depends on different values of $r$ and $\eps$. To obtain the optimal solution set for $\eps\in(n-k,n)$ in Lemma~\ref{lem:opti}, we need a complete (and complicated) calculation of the maximum among $A,B,\max_{{i\in I}} K_i$ defined as in the following lemma. Hence, for a clear demonstration in the proof of Lemma~\ref{lem:opti}}, we extract {this} complicated calculation {here}.
\begin{lemma}
Let $n,k\in\N, n\ge k, r>0$ and $\eps\in(n-k,n)$. Let $r_1=\frac{\lceil \eps \rceil-(n-k)}{\eps-(n-k)},r_2=\frac{1}{\eps+1-\lceil\eps\rceil},$ $A:=2n-k-\lceil\eps\rceil,B:=n+r(n-k-\eps),$ $K_i:=2n-k-r\eps+(r-1)i$ for {$i\in[n-k+1..\lceil \eps \rceil-1]=:I$ (note that $I=\emptyset$ if $n-k+1>\lceil \eps \rceil-1$, that is, if $\eps\in (n-k,n-k+1]$), and $C:=K_{\lceil \eps \rceil-1}$}. Then $1\le r_1\le r_2$  { holds, where} $r_1=r_2$ iff $\eps\in [n-k+1..n-1] \cup (n-k,n-k+1)$, and $r_1=1$ iff $\eps\in [n-k+1..n-1]$.
Besides, we have
\begin{equation*}
\begin{split}
\max&{}\left\{A,B,\max_{{i\in I}} K_i\right\}\\
&{}={}\begin{cases}
B, &\text{if $(\eps\in(n-k,n),r\in(0,1)),$}\\
K_{{i\in I}}=B, &\text{if $(\eps\in(n-k,n)\setminus \N, r=1)$,}\\
K_{{i\in I}}=B=A, &\text{if $(\eps\in(n-k,n)\cap\N, r=1)$,}\\
C=B, &\text{if $(\eps\in(n-k,n-k+1],r\in(1,r_1)),$}\\
C, &\text{if $(\eps\in(n-k+1,n),r\in(1,r_1)),$}\\
C=B=A, &\text{if $(\eps\in(n-k,n-k+1),r=r_1{>1}),$}\\
C, &\text{if $(\eps\in(n-k+1,n)\setminus \N,r=r_1{>1}),$}\\
C, &\text{if $(\eps\in(n-k,n),r\in(r_1,r_2)),$}\\
C=A, &\text{if $(\eps\in(n-k+1,n)\setminus \N,r=r_2{>r_1}),$}\\
A, &\text{if $(\eps\in(n-k,n),r>r_2)$.}\\
\end{cases}
\end{split}
\end{equation*}
{Note here that the above calculation is complete. For example, ``$K_{i\in I}=B, \text{if $(\eps\in(n-k,n)\setminus \N, r=1)$}$'' means that when $\eps\in(n-k,n)\setminus \N$ and $r=1$, 
$$
\max\left\{A,B,\max_{i\in I} K_i\right\}=K_{i\in I}=B>A
$$
for any $i\in I$.}
\label{lem:max3}
\end{lemma}
\begin{proof}
It is not difficult to see that $r_1\ge 1$ and the equality holds iff $\eps\in \N$. Since $\lceil\eps\rceil\ge n-k+1$, we know that $\eps-(n-k) \ge \eps+1-\lceil\eps\rceil >0$, and thus
\begin{align*}
r_2-r_1&{}={}\frac{1}{\eps+1-\lceil\eps\rceil}-\frac{\lceil \eps \rceil-(n-k)}{\eps-(n-k)} \\
&{}={} \frac{\lceil\eps\rceil-\eps}{\eps+1-\lceil\eps\rceil}-\frac{\lceil\eps\rceil-\eps}{\eps-(n-k)} \ge 0,
\end{align*}
where the equality holds iff $\eps=\lceil \eps \rceil$ or $\lceil \eps \rceil=n-k+1$, that is, $\eps \in \N\cup (n-k,n-k+1]$. Thus $r_2\ge r_1\ge 1$, and  {we also know that} $r_2=r_1=1$ holds iff $\eps \in \N$, and $r_2=r_1>1$ holds iff $\eps \in (n-k,n-k+1)$.

We first note that
\begin{equation}
    \begin{split}
        K_{\lceil \eps \rceil -1}
        &{}=2n-k-r\eps+(r-1)(\lceil \eps\rceil-1)\\
&={}2n-k-(\lceil \eps\rceil-1)+r(\lceil \eps \rceil-1-\eps)=C,
    \end{split}
    \label{eq:kc}
\end{equation}
and that
\begin{equation}
\begin{split}
B-C={}&{}n+r(n-k-\eps)
-(2n-k-(\lceil \eps\rceil-1)
+r(\lceil \eps \rceil-1-\eps))\\
={}&{}k+\lceil \eps\rceil-1-n+r(n-k-\eps-(\lceil \eps \rceil-1-\eps))\\
={}&{}k+\lceil \eps\rceil-1-n+r(n-k-\lceil \eps \rceil+1)\\
={}&{}(r-1)(n-k-\lceil \eps \rceil+1).
\end{split}
\label{eq:comp}
\end{equation}
For $r<r_1=\frac{\lceil \eps \rceil-(n-k)}{\eps-(n-k)}$, we have 
\begin{equation}
\begin{split}
B&{}={}n+r(n-k-\eps)
>{}n+ \frac{\lceil \eps \rceil-(n-k)}{\eps-(n-k)} (n-k-\eps)\\
&{}={}n+n-k-\lceil \eps \rceil=2n-k-\lceil \eps \rceil=A,
\end{split}
\label{eq:comp12}
\end{equation}
where the first inequality uses $\eps>n-k$. In this case, we further consider $r\in(0,1),r=1,$ and $r\in(1,r_1)$ respectively. If $r\in(0,1)$, then 
\begin{equation*}
\begin{split}
\max_{{i\in I}}&{} K_i
=\max_{{i\in I}} \left(2n-k-r\eps+(r-1)i\right)\\
&{}={}2n-k-r\eps+(r-1)(n-k+1)
=n-1+r(n-k+1-\eps)\\
&{}=n+r(n-k-\eps)+r-1
<n+r(n-k-\eps)=B.
\end{split}
\label{eq:r01}
\end{equation*}
With (\ref{eq:comp12}), {the result for $r \in (0, 1)$ is shown}.

If $r=1$, then 
\begin{equation*}
\begin{split}
\max_{{i\in I}}K_i
&{}=\max_{{i\in I}} \left(2n-k-r\eps+(r-1)i\right)\\
&{}=K_{{i\in I}}=2n-k-\eps~(=n+r(n-k-\eps)=B)\\
&{}\ge 2n-k-\lceil\eps\rceil=A,
\end{split}
\label{eq:r1}
\end{equation*}
where the equality in the last inequality holds iff $\eps\in\N$. Hence, we have the results for $r=1$.

If $r\in (1,r_1)$, which already requires $\eps \notin \N$ to ensure $r_1>1$ as we proved before, then 
\begin{equation}
\begin{split}
&{}\max_{{i\in I}}K_i
=\max_{{i\in I}} \left(2n-k-r\eps+(r-1)i\right)
=C\ge B,
\end{split}
\label{eq:rb1}
\end{equation}
where the last equality uses (\ref{eq:kc}), and the last inequality uses $(\ref{eq:comp})\le 0$ since $r>1$ and $\lceil\eps\rceil\ge n-k+1$, and thus the equality in this inequality holds iff $\lceil\eps\rceil= n-k+1$, that is, $\eps\in(n-k,n-k+1]$. Hence, we have the results for $r\in(1,r_1)$. Note also that (\ref{eq:rb1}) holds for all $r>1$.

For $r=r_1=\frac{\lceil \eps \rceil-(n-k)}{\eps-(n-k)}$, we have 
\begin{align*}
B=n+r(n-k-\eps)=2n-k-\lceil \eps \rceil=A.
\end{align*}
With (\ref{eq:rb1}) and {$r_1>1$,}
we have the results for $r=r_1$.

For $r>r_1=\frac{\lceil \eps \rceil-(n-k)}{\eps-(n-k)}$, similar to (\ref{eq:comp12}), we have 
\begin{align}
B<A.
\label{eq:c12rbr1}
\end{align}
In this case, we further consider $r\in(r_1,r_2),r=r_2,$ and $r>r_2$ respectively. If $r\in(r_1,r_2)$, which already requires $\eps \in (n-k+1,n)\setminus\N$ to ensure $r_2>r_1$ as we proved before, then {with (\ref{eq:kc}) we have}
\begin{equation*}
    \begin{split}
        {C} &{}> 2n-k-(\lceil \eps\rceil-1)+\frac{\lceil \eps \rceil-1-\eps}{\eps+1-\lceil\eps\rceil}
=2n-k-\lceil \eps\rceil=A,
    \end{split}
    \label{eq:rr1r2}
\end{equation*}
where the first inequality uses $\lceil \eps \rceil-1-\eps<0$ and $r<r_2=\frac{1}{\eps+1-\lceil\eps\rceil}$. Together with (\ref{eq:c12rbr1}), the result for $r\in(r_1,r_2)$ is proved. 

Similarly, if $r=r_2$, we have
{$C$ }$=A$.
With (\ref{eq:c12rbr1}) and {$r_2>r_1$,}
we have the result for {$r = r_2$}.

Also, if $r>r_2$, we have
{$C$ }$<A$. With (\ref{eq:c12rbr1}), we have the result for $r>r_2$.
%
%
\end{proof}

Now we have the following lemma for the optimal solution set for different penalty coefficients $r$, and different relationships between $\eps$ and $n$.
\begin{lemma}
Let {$A,B,C,I$ and $K_i,i\in I$} be defined {as} in Lemma~\ref{lem:max3}. Let ${\eps'=\lceil\eps\rceil-1}$ and {$D_i=\{x\mid f^k(x)=(2n-k-i,i)\}$} {(that is, $D_i=\{x\mid x_{[1..n-k]}=1^{n-k}, |x_{[n-k+1..n]}|_1=n-i\}$)} for $i=n-k,\dots,n$. Let the penalty coefficient $r>0$, and let $S$ be the optimal solution set of maximizing (\ref{eq:epf}). Then
\begin{equation*}
S=\begin{cases}
D_n, &\text{if $(\eps \ge n, r>1)$},\\
\cup_{i=n-k}^n D_i, &\text{if $(\eps \ge n, r=1)$},\\
D_{n-k}, &\text{if $\eps\le n-k$ or
$(\eps \in (n-k,+\infty), r\in(0,1))$},\\
\cup_{i=n-k}^{\lceil\eps\rceil-1}D_i, &\text{if $(\eps \in (n-k,n), \eps\notin \N, r=1)$},\\
\cup_{i=n-k}^{\lceil\eps\rceil}D_i, &\text{if $(\eps \in (n-k,n), \eps\in \N, r=1)$},\\
D_{n-k},&\text{if $(\eps \in (n-k,n-k+1], r\in(1,r_1))$},\\
D_{\eps'},&\text{if $(\eps \in (n-k+1,n), r\in(1,r_1))$ or}\\
&\text{ $(\eps \in (n-k+1,n)\setminus \N, r=r_1{>1})$ or}\\
&\text{ $(\eps \in (n-k,n),r\in(r_1,r_2)),$}\\
D_{n-k}\cup D_{\lceil\eps\rceil}, &\text{if $(\eps\in(n-k,n-k+1),r=r_1{>1})$},\\
D_{\eps'}\cup D_{\lceil\eps\rceil}, &\text{if $(\eps\in(n-k+1,n)\setminus \N,r=r_2{>r_1})$},\\
D_{\lceil\eps\rceil}, &\text{if $(\eps\in(n-k,n),r>r_2)$}.
\end{cases}
\end{equation*}
\label{lem:opti}
\end{lemma}
\begin{proof}
If $\eps \ge n$, since $\max_{x\in\{0,1\}^n}f_2^k(x)=f_2^k(1^{n-k}0^k)=n$, we know that for any $x$,
\begin{equation}
\begin{split}
g(x)={}&{}f_1^k(x)+r(f_2^k(x)-\eps)
=f_1^k(x)+rf_2^k(x)-r\eps\\
={}&{}\sum_{j=1}^n x_j +r\left( \left(\sum_{j=1}^{n-k} x_j\right) + \left(\sum_{j=n-k+1}^n 1-x_j\right)\right) 
-r\eps\\
={}&{}(1+r)\left(\sum_{j=1}^{n-k} x_j\right)+(1-r)\left(\sum_{j=n-k+1}^n x_j\right) 
+r(k-\eps).
\end{split}
\label{eq:cal}
\end{equation}
Then it is not difficult to see the set of optimal solutions is $\{1^{n-k}0^k\}=D_n$ if $r>1$, $\cup_{i=n-k}^nD_i$ if $r=1$, and $D_{n-k}$ if $r\in(0,1)$.

If $\eps \in (n-k,n)$, then
\begin{align*}
g(1^n)=n+r\min\{0,n-k-\eps\}=n+r(n-k-\eps)=B.
\end{align*}
Recall that $\max_{x\in\{0,1\}^n}f_2^k(x)=f_2^k(1^{n-k}0^k)=n$. 
Note that for $i \in [n-k..n]$,
\begin{equation}
\begin{split}
\max_{x\in\{z\mid f_2^k(z)=i\}}f_1^k(x)&{}=n-k+(k-(i-(n-k)))\\
&{}=n+(n-k-i)=2n-k-i
\end{split}
\label{eq:dec}
\end{equation}
decreases when $i$ increases. Hence, we know that for any $x\in \{0,1\}^n$ with $f_2^k(x)\ge \eps$, we have
\begin{align}
g(x)&{}={}f_1^k(x) \le 2n-k-\lceil\eps\rceil=A,
\label{eq:aeps}
\end{align}
where {under the condition $f_2^k(x)\ge \eps$, $g(x)=A$} holds {iff} {${x\in\{z\mid f^k(z)=(2n-k-\lceil \eps \rceil,\lceil\eps\rceil)\}=D_{\lceil \eps \rceil}}$}.
Also note that for $i\in[0..n-k]$,
$$\max_{x\in\{z\mid f_2^k(z)=i\}}f_1^k(x)=i+k$$ increases when $i$ increases. Hence, together with (\ref{eq:dec}) we know that for any $x\in \{0,1\}^n$ with $f_2^k(x)< \eps$, we have
\begin{equation}
\begin{split}
g(x){}&{}=f_1^k(x) +r(f_2^k(x)-\eps) \\
\le{}&{} \max\bigg\{\max_{i\in[0..n-k]}\left(i+k+r(i-\eps)\right),
\max_{{i\in I}} \left(2n-k-i+r(i-\eps)\right)\bigg\}\\
{=}~{}&{} \max\left\{n-k+k+r(n-k-\eps),
\max_{{i\in I}} K_i\right\}\\
={}&{}\max\left\{B,\max_{{i\in I}} K_i\right\},
\end{split}
\label{eq:leps}
\end{equation}
where the {first} equality {uses the definition of $K_i$ and the fact that the maximum of $i+k+r(i-\eps)$ for $i\in[0..n-k]$ is reached iff $i=n-k$}, that is, $x=1^n$. 
With (\ref{eq:aeps}) and (\ref{eq:leps}), we know that the {maximal} function value {of $g$} is
\begin{align*}
\max\left\{A,B,\max_{{i\in I}} K_i\right\}.
\end{align*}
Note that in the above discussion, we require {the only case} $D_{\lceil \eps \rceil}$ to let $A$ be reachable, $D_{n-k}$ for $B$, and $D_i$ for $K_i$ with $i\in[n-k+1..\lceil\eps\rceil-1]$ (including $D_{\eps'}$ for $C$).
Then from Lemma~\ref{lem:max3}, the case for $\eps\in(n-k,n)$ is proved.

If $\eps \le n-k$, then
\begin{align*}
g(1^n)=n+r\min\{0,n-k-\eps\}=n.
\end{align*}
For any $x\in \{0,1\}^n\setminus\{1^n\}$, we know that $f_1^k(x) < n$, and thus
\begin{equation*}
g(x)=f_1^k(x)+r\min\{0,f_2^k(x)-\eps\} \le f_1^k(x) < n=g(1^n).
\end{equation*}
That is, the optimal solution is $1^n$.
\end{proof}

Recall that for a certain $\eps\in(n-k,n]$, the optimal solution set of (\ref{eq:con}) is $\{x\mid f^k(x)=(2n-k-\lceil\eps\rceil,\lceil\eps\rceil)\}=D_{\lceil \eps \rceil}$, and the optimal set is $\{x\mid f^k(x)=(n,n-k)\}=D_{n-k}$ for $\eps\le n-k$ and is $\emptyset$ for $\eps>n$. 
We have the following corollary for the comparison between the constrained problem (\ref{eq:con}) and the problem (\ref{eq:epf}) constructed via the exterior penalty function.

\begin{corollary}
Let $r>0$. Then the optimal solution sets of (\ref{eq:con}) and (\ref{eq:epf}) are identical for $\eps\in (-\infty,n-k]$, {$(\eps=n,r>1)$,} and $(\eps\in(n-k,n),r\in(\frac{1}{\eps+1-\lceil\eps\rceil},+\infty))$, and are different for other cases.
\label{cor:conext}
\end{corollary}

{\noindent{B. Runtime Analysis}}\\
{B1. Improper Parameter Settings {for an Incomplete Coverage of the Pareto Front}}

With the above result, we can easily have the following corollary about the algorithm's difficulty in solving the exterior penalty problem with improper parameter settings{, which results in an incomplete coverage the Pareto front of $\omm_k$}.
\begin{corollary}
Let {$M$ be defined as in Lemma~\ref{lem:PF},} $k\ge {1},r\in(0,1)$, and {$S\subset \R$ be a set of $\eps$. Let $x_{{\eps}}$ be one optimal solution for the maximization of $g$}
with a certain $\eps\in S$.
Let $F=\{f^k(x_{{\eps}})\mid \eps\in S\}$. Then {$M \not\subseteq F$}.
\label{cor:r01}
\end{corollary}
\begin{proof}
From Lemma~\ref{lem:opti}, we know that {for this case of $r\in(0,1)$,} $x_{{\eps}}\in D_{n-k}=\{x\mid f^k(x)=(n,n-k)\}$. Hence, ${F=\{(n,n-k)\}}$ and thus {$|F|=1<2\le k+1=|M|$}. Then {$M \not\subseteq F$}.
\end{proof}

{\noindent{B2. Runtime for Proper Parameter Settings}}

Luckily, from Lemma{s~\ref{lem:PF} and}~\ref{lem:opti}, the following corollary shows that with careful settings, this approach can result in a full coverage of {the} Pareto front for (\ref{eq:coczk}).
\begin{corollary}
Let {$\eps_0 \in (-\infty, n-k],\eps_1\in (n-k,n-k+1], \dots,\eps_{k-1}\in (n-2,n-1]$, and $\eps_{k} \in (n-1,+\infty)$}. {Let $M$ be defined as in Lemma~\ref{lem:PF}}.
\begin{itemize}
    \item Consider to maximize (\ref{eq:epf}) with {$\eps=\eps_i$ and} $r>\frac{1}{\eps_i+1-\lceil\eps_i\rceil}$ {for any} $ i\in[0..k]$. Then $D_{n-k+i}$ is the set of optimal solutions.
    \item {Let $G$ be any set with $\{\eps_0,\dots,\eps_k\}\subseteq G$ and} $r>\max_{i=0,\dots,k}\frac{1}{\eps_i+1-\lceil\eps_i\rceil}$. {Let $x_{{\eps}}$ be one optimal solution for the maximization of $g$ with a certain $\eps\in G$ and $F=\{f^k(x_{{\eps}})\mid \eps\in G\}$. Then $M \subseteq F$.}
\end{itemize} 
\label{cor:cor}
\end{corollary}

With the above setting, now we analyze the runtime of the randomized local search algorithm  {to maximize (\ref{eq:epf}) with $r>\frac{1}{\eps+1-\lceil\eps\rceil}$}. {The randomized local search (RLS) algorithm is a simple single-objective optimizer. It starts from a single randomly generated solution. In each generation, it applies the one-bit mutation (uniformly at random picking one bit and flipping its bit value) on the parent solution to generate an offspring solution. If this offspring has at least the same fitness as its parent, then the offspring will replace its parent and enter into the next generation.} We first have the following two lemmas about the survival situations starting from an infeasible solution and from a feasible solution, respectively.
\begin{lemma}
Consider using the randomized local search algorithm to maximize {$g $ defined as in (\ref{eq:epf}) with $r>\frac{1}{\eps+1-\lceil\eps\rceil}$}. Then starting from an infeasible solution (whose $f_2^k$ value is less than $\eps$), only the offspring via flipping one $0$ bit in the first $n-k$ positions or via flipping one $1$ bit in the last $k$ positions will survive to the next population. 
\label{lem:infea}
\end{lemma}
\begin{proof}
Let $x$ be an infeasible solution in the current population, $y$ be the offspring generated by applying one-bit mutation to $x$, and {$j$} be the bit position where $x$ and $y$ are different. Let $I_0, I_1$ be {the sets of all} bit positions with the value of zero and the value of one in the first $n-k$ bits of $x$ respectively, and let $I_0',I_1'$ {be the sets of all} bit positions with the value of zero and the value of one in the last $k$ bits of $x$ respectively. If {$j\in I_0$}, then $f_1^k(y)>f_1^k(x), f_2^k(y)>f_2^k(x)$, and thus $g(y) > g(x)$. Hence $y$ will survive to the next generation. 

If {$j\in I_1$}, then $f_1^k(y)<f_1^k(x),f_2^k(y)<f_2^k(x)$, and thus $g(y) < g(x)$. Hence $y$ will be removed.

If {$j\in I_0'$}, then $f_1^k(y)=f_1^k(x)+1$ and ${f_2^k(y)=f_2^k(x)-1<{\eps}}$ as $x$ is infeasible. Then
\begin{align*}
g(y)-g(x)
=f_1^k(y)+r(f_2^k(y)-{\eps}) -(f_1^k(x)+r(f_2^k(x)-{\eps}))
=1-r<0,
\end{align*}
where the last inequality uses {$r>\frac{1}{\eps+1-\lceil\eps\rceil}>1$}. Hence, $y$ will be removed.

If {$j\in I_1'$}, then $f_1^k(y)=f_1^k(x)-1$ and $f_2^k(y)=f_2^k(x)+1$. Then if $y$ is infeasible, we have
\begin{align*}
g(y)-g(x)
= f_1^k(y)+r(f_2^k(y)-{\eps}) -(f_1^k(x)+r(f_2^k(x)-{\eps}))
=r-1>0,
\end{align*}
where the last inequality uses {$r>\frac{1}{\eps+1-\lceil\eps\rceil}>1$}. If $y$ is feasible, then
\begin{align*}
g(y)-g(x)&{}={} f_1^k(y) -(f_1^k(x)+r(f_2^k(x)-{\eps}))
=r({\eps}-f_2^k(x))-1\\
&{}\ge{} r({\eps}-\lceil {\eps} \rceil + 1)-1>0,
\end{align*}
where the first inequality uses $f_2^k(x)\le \lceil {\eps} \rceil-1$ as $x$ is infeasible and $f_2^k(x)$ is an integer, and the last inequality uses $r>\frac{1}{{\eps}+1-\lceil{\eps}\rceil}$ and ${\eps}-\lceil {\eps} \rceil + 1>0$.
 Hence, $y$ will survive to the next generation.
\end{proof}

\begin{lemma}
Consider using the randomized local search algorithm to maximize {$g$ defined as in (\ref{eq:epf}) with $r>\frac{1}{\eps+1-\lceil\eps\rceil}$}. Then once a feasible solution is reached, only the offspring via flipping one $0$ bit in the last $k$ positions (conditional on that the current solution has its $f_2^k$ value at least $\lceil{\eps}\rceil+1$), or via flipping one $0$ bit in the first $n-k$ positions will survive to the next population. Besides, all solutions in future generations are feasible.
\label{lem:fea}
\end{lemma}
\begin{proof}
Let $x$ be a feasible solution in the current population and $y$, {$j$}, $I_0,I_1,I_0'$ and $I_1'$ be defined {as} in the proof of Lemma~\ref{lem:infea}.
If {$j\in I_0$}, then $f_1^k(y)>f_1^k(x)$ and $f_2^k(y)>f_2^k(x)\ge$ {$\eps$}. Thus $y$ is feasible and $g(y)>g(x)$, and then this feasible $y$ will survive to the next generation.

If {$j\in I_1$}, then $f_1^k(y)<f_1^k(x)$ and $f_2^k(y)<f_2^k(x)$. Thus, $g(y)<g(x)$ and $y$ will be removed.

If {$j\in I_0'$}, then $f_1^k(y)=f_1^k(x)+1$ and $f_2^k(y)=f_2^k(x)-1$. If $y$ is feasible, then $f_2^k(x)\ge \lceil${$\eps$}$\rceil+1$, $g(y)=g(x)+1$, and $y$ will survive to the next generation. If $y$ is infeasible, that is, $f_2^k(y)< \eps_i$, then $f_2^k(y)\le \lceil${$\eps$}$\rceil-1$ as $f_2^k(y)$ is an integer. {In this case,}
\begin{align*}
g(y)-g(x)&=f_1^k(y)+r(f_2^k(y)-{\eps})-f_1^k(x)
\le 1+r(\lceil{\eps}\rceil-1-{\eps}) < 0,
\end{align*}
where the last inequality uses {$r>\frac{1}{\eps+1-\lceil\eps\rceil}$} and {$\lceil\eps\rceil-1-\eps<0$}. Hence, such $y$ will be removed in the selection.

If {$j\in I_1'$}, then $f_1^k(y)=f_1^k(x)-1$ and $f_2^k(y)=f_2^k(x)+1 >$ {$ \eps$}. Then $y$ is feasible and $g(y)=f_1^k(y)<f_1^k(x)=g(x)$. Hence, such $y$ will be removed in the selection.
\end{proof}

Now we state the runtime of RLS to maximize the constructed problem (\ref{eq:epf}). 

\begin{theorem}
Consider using the randomized local search algorithm to maximize {$g$ defined as in (\ref{eq:epf}) with $r>\frac{1}{\eps+1-\lceil\eps\rceil}$}. Then after $O(n\ln n)$ iterations in expectation, an optimal solution is reached.
\label{thm:rlsom}
\end{theorem}
\begin{proof}
We first consider the case for {$\eps$} $\in(n-k,n]$.

We reuse the notations of $I_0,I_1,I_0',$ and $I_1'$ in the proof of Lemma~\ref{lem:infea}. Recalling Lemma~\ref{lem:infea}, we know that before a feasible solution is reached for the first time, only the offspring via flipping one bit in $I_0\cup I_1'$ will be accepted into the next population. Note that for any $x$, $|I_0\cup I_1'|=n-f_2^k(x)$. Hence, a solution is infeasible iff $|I_0\cup I_1'| \in [n-\lceil${$\eps$}$\rceil+1..n]$ as an infeasible solution $x$ has $f_2^k(x)\in[0..\lceil${$\eps$}$\rceil-1]$. Therefore, the expected number of iterations to reach a feasible solution for the first time is at most
\begin{align}
\sum_{i={n-\lceil \eps \rceil+1}}^{{n}} \frac{n}{i} \le n(\ln n+1).
\label{eq:ft}
\end{align}

Now we consider the case after a feasible solution is reached for the first time. {For a feasible solution, we know that $|I_1|+|I_0'|\ge \eps$ and further $|I_1|+|I_0'|\ge \lceil\eps\rceil$ as $|I_1|+|I_0'|\in \N$. Then from $|I_1|+|I_0|=n-k$, we have
\begin{align}
    |I_0'|\ge \lceil\eps\rceil-|I_1|=\lceil\eps\rceil-n+k+|I_0|\ge \lceil\eps\rceil-n+k.
\label{eq:i0}
\end{align}
}
{Recalling} Lemma~\ref{lem:fea}, we know that only the offspring via flipping one bit in $I_0$ or via flipping one bit in $I_0'$ (conditional on {the event, denoted as $A$, that} the parent has $f_2^k$ value of at least $\lceil ${$\eps$}$ \rceil+1$) will be accepted into the next population{, and all solutions in the future generations are feasible. Hence, in all future generations, both $|I_0|$ and $|I_0'|$ will not increase, the replacement can still happen if $|I_0|>0$), and (\ref{eq:i0}) always holds. Therefore, $|I_0\cup I_0'|=\lceil${$\eps$}$\rceil-n+k$ happens iff $(|I_0|=0,|I_0'|=\lceil\eps\rceil-n+k)$, which is only reached by}
the optimal set {$D_{\lceil \eps\rceil}=\{x\mid f(x)=(2n-k-\lceil\eps\rceil,\lceil\eps\rceil)\}$}.
Hence, to reach an optimum, we need to decrease $|I_0\cup I_0'|$ {to $\lceil\eps\rceil-n+k$.}
{From $|I_0'|\ge\lceil\eps\rceil-n+k+|I_0|$ in (\ref{eq:i0}), we have}
\begin{align*}
|I_0|+|I_0'|&{}\le{} n-k+|I_0'|-\lceil {\eps}\rceil + |I_0'| = n-k-\lceil {\eps}\rceil+2|I_0'|\\
&{}\le{} n-k-\lceil {\eps}\rceil+2k=n+k-\lceil {\eps}\rceil,
\end{align*}
where 
the last inequality uses $|I_0'| \le k$. 
Hence, the expected number of iterations to reach an optimal solution for the first time is at most
\begin{align}
\sum^{{n+k-\lceil \eps \rceil}}_{j={\lceil \eps\rceil-n+k+1}} \frac{n}{j} \le n(\ln (n+k-\lceil {\eps} \rceil)+1).
\label{eq:ro}
\end{align}
Together with (\ref{eq:ft}), we have $O(n \ln n)$ expected number of iterations for reaching an optimum for {$\eps$}$\in(n-k,n]$.

{Now we consider the case for $\eps>n$. Recalling Lemma~\ref{lem:opti}, we know that the optimal solution is $1^{n-k}0^k$, that is, when $|I_0|=|I_1'|=0$. Note that all solutions are infeasible as $f_2^k\le n$, and recall that (from Lemma~\ref{lem:infea}) only the offspring that decreases $|I_0|$ or $|I_1'|$ will be accepted into the next population. Thus the the expected number of iterations to reach the optimal solution $1^{n-k}0^k$ is at most
\begin{align*}
    \sum_{i=1}^n \frac{n}{i}\le n(\ln n+1),
\end{align*}
where we pessimistically consider $|I_0\cup I_1'|$ decreasing from $n$ to $0$.}

{Now it remains to discuss the case for $\eps\le n-k$. Recalling Lemma~\ref{lem:opti}, we know that the optimal solution is $1^{n}$, that is, when $|I_0|=|I_0'|=0$. If the initial solution is infeasible, then (\ref{eq:ft}) shows that at most $n(\ln n+1)$ expected iterations are needed to reach a feasible solution for the first time. Thus, now we only consider that we start from a feasible solution. With additional restriction $|I_0|+|I_0'|\le n$, analogous to (\ref{eq:ro}), we know that the additional expected number of iterations to reach an optimal solution is at most
\begin{align*}
    \sum_{i=1}^{\min\{n,n+k-\lceil \eps \rceil\}} \frac{n}{i}\le n(\ln n+1).
\end{align*}
}

Therefore, this theorem is proved.
\end{proof}

Since there are {$k+1$} Pareto front points, with Corollary~\ref{cor:cor} {and} Theorem~\ref{thm:rlsom}{, we have the following corollary to show the total runtime to cover} the full Pareto front for {$\omm_k$ when $\eps$ and $r$ are carefully chosen}.

{\begin{corollary}
    Let $\eps_0,\dots,\eps_k$ be defined as in Corollary~\ref{cor:cor}, $G=\{\eps_0,\dots,\eps_k\}$, and $r>\max_{i=1,\dots,k}\frac{1}{\eps_i+1-\lceil\eps_i\rceil}$. Consider using the randomized local search algorithm to maximize $g$ with $\eps\in G$ one by one (which terminates when the maximum of corresponding $g$ is reached for the first time), and let $P$ be set of all obtained solutions. Then within expected $O(\max\{k,1\} n\ln n)$ fitness evaluations, $P$ will have its size of $k+1$ and cover the full Pareto front.
\end{corollary}
}

\subsection{Nonparameter Penalty Function}
Another way for handling the constrained problem (\ref{eq:con}) is by nonparameter penalty function~\cite{Deb00,ZhouH07}. Instead of introducing a parameter to penalize the violated constraint in the exterior way discussed before, this nonparameter way adds the amount of violence directly to the objective. Formally, the nonparameter penalty way transfers the constrained problem (\ref{eq:con}) into maximizing the following function.
\begin{definition}
The nonparameter penalty function w.r.t. the constrained problem (\ref{eq:con}) $g: \{0,1\}^n\rightarrow \R$ is defined by
\begin{equation}
g(x)=
\begin{cases}
f^k_1(x), & \text{if $x$ is feasible}\\
f^k_1(x)+f^k_2(x)-\eps, & \text{if $x$ is infeasible} \\
\end{cases}
\label{eq:npf}
\end{equation}
for $x\in\{0,1\}^n$.
\end{definition}
It is not difficult to see that the above (\ref{eq:npf}) is identical to (\ref{eq:epf}) with $r=1$. Recalling Lemma~\ref{lem:opti}, we know that the optimal set for maximizing (\ref{eq:npf}) is
\begin{align*}
\begin{cases}
\cup_{i=n-k}^n D_i, & \eps\ge n\\
\cup_{i=n-k}^{\lceil\eps\rceil-1}D_i, &\eps\notin \N, \eps\in(n-k,n)\\
\cup_{i=n-k}^{\lceil\eps\rceil}D_i, &\eps\in \N, \eps\in(n-k,n)\\
D_{n-k}, &\eps\le n-k{.}
\end{cases}
\end{align*}
{As discussed in Section~\ref{ssec:epsc}, the optimal set of (\ref{eq:con}) is $\emptyset$ for $\eps >n$, $D_{\lceil \eps \rceil}$ for $\eps\in(n-k,n]$, and $D_{n-k}$ for $\eps\le n-k$. }Hence, it is different from the optimal set of the problem (\ref{eq:con}){, especially for $\eps\in (n-k,n]$. For the setting that one $\eps$ value corresponds to one optimum solved by one algorithm, a specific $\eps\in (n-k,n)\setminus \N$ will return a solution in $D_j$ with the random variable $j\in[n-k..\lceil\eps\rceil-1]$ that follows an unknown distribution depending on the algorithm to solve (also for $\eps\ge n$ and $\eps\in (n-k,n)\cap \N$)}. Thus, the optimal solution set{, say $U$,} of (\ref{eq:npf}) with different values of $\eps$ {does not naturally ensure that
$U\cap D_i\ne \emptyset$ holds for all $i=n-k,\dots,n$. That is, it does not naturally} guarantee the full coverage of the Pareto front.

\subsection{Summary and Comment}
From this section, we see that although the $\eps$-constraint approach can cover the full Pareto front for $\omm_k$, the solving modes have difficulties or inconveniences. With careful settings of $\eps$ and penalty coefficient $r$, the exterior penalty function way can cover the full Pareto front with expected $O(\max\{k,1\}n\ln n)$ function evaluations. However, the nonparameter penalty function way cannot guarantee that the optimal solution sets cover the full Pareto front.

We will also note that the above analysis shows that careful penalty design of solving modes is needed for constrained optimization.

\section{Multiobjective Evolutionary Algorithms}\label{sec:moeas}
In this section, we will show that the multiobjective evolutionary algorithms can easily cover the full Pareto front compared with the difficulty or the inconvenience witnessed in the typical approaches in Section~\ref{sec:typ}.

\subsection{MOEA/D}
The MOEA/D is one kind of MOEAs that decomposes the multiobjective problem into several single-objective problems and solves them in a co-evolutionary way. Thus, it shares a similarity to the methods in Section~\ref{sec:typ} where also several single-objective problems are considered. Here, we first give a brief introduction of the MOEA/D and then discuss its runtime complexity. 

\subsubsection{Algorithm Description}
The MOEA/D is first proposed by~\cite{ZhangL07} and the first theoretical runtime results are obtained by~\cite{LiZZZ16}. It first employs a decomposition method to decompose the multiobjective problem into {the  {optimization} of} $H+1$ subproblems. The Tchebycheff decomposition approach used in~\cite{ZhangL07} constructs the $i$-th {($i=0,\dots,H$)} subproblem  {(to be minimized)} $h_i(x)$ with a predefined evenly spread weight $w_i$ by
\begin{equation*}
\begin{split}
h_i(x)&={}\max\{w_i|f_1^k(x)-z_1^*|,(1-w_i)|f_2^k(x)-z_2^*|\},
\end{split}
\end{equation*}
where $z_i^*$ in the reference point {$z^*=(z_1^*,z_2^*)$} is {usually set to be the best $f_i$} value {found so far by the algorithm}.

After the decomposition, the algorithm initializes the population{, sets $z^*=(z_1^*,z_2^*)$ with $z_i^*$ the best value of the $i$-th objective,} and puts the mutually incomparable solutions into the external population $P_e$. For each subproblem, one offspring is generated {and $z^*$ is updated accordingly}. In its co-evolutionary solving mode,  {with predefined neighborhood size $T$, the $i$-th subproblem has $T$ subproblem neighbors with indices $\{i_1,\dots,i_T\}:=B_i$ such that $w_{i_1},\dots,w_{i_T}$ are $T$ closest weight vectors to $w_i$.} Then for each individual with the index of its corresponding subproblem in $B_i$, the individual itself and its fitness {will be replaced} by the generated offspring if this offspring has a better or the same $h_i$ value. Later $P_e$ is updated. The whole procedure is shown in Algorithm~\ref{alg:moead}.
\begin{algorithm}[tb]
    \caption{MOEA/D for  {maximizing} the bi-objective function $f=(f_1,f_2)$}
\textbf{Parameters:} $H+1$: the number of the subproblems. $w_0,\dots,w_H$: the weight for the subproblems. $T$: the number of the subproblems to construct the neighbor of a given subproblem.
    \begin{algorithmic}[1]
    \STATE Construct $h_i$ subproblem according to $w_i$ for all $i=0,\dots, H$, and  {calculate $B_i=\{i_1,\dots,i_T\}$ such that $w_{i_1},\dots,w_{i_T}$ are closest weight vectors to $w_i$}.
    \STATE Initialize the external population $P_e=\emptyset$. Initialize the population $P=\{x_0,\dots,x_H\}$, evaluate $f(x_i)$ and $h_i(x_i)$ for $i=0,\dots, H$, and set the reference point {$z^*=(\max\{f_1(z)\mid z\in P\},\max\{f_2(z)\mid z\in P\})$.} For $i\in[0..H]$, let $P_e=\{z\in P_e \mid z \npreceq x_i\} \cup \{x_i\}$ if $x_i$ is not dominated by any solution in $P_e$. 
    \FOR {$g=1,2,\dots$}
    \FOR {$i=0,\dots, H$}
    \STATE {Generate $x_i'$ via applying the mutation to $x_i$.}
    \STATE {For each $j=1,2$, if $f_j(x_i')>z^*_j$ then $z^*_j=f_j(x_i')$.}
    \STATE {For each {$j\in B_i$} such that $h_{j}(x_i')\le h_{j}(x_j)$}, set $x_j=x_i'$.
    \STATE {If there is no $z \in P_e$ such that $x_i' \prec z$ then $P_e=\{z\in P_e \mid z \npreceq {x_i'}\} \cup \{{x_i'}\}$.}
    \ENDFOR
    \ENDFOR
    \end{algorithmic}
    \label{alg:moead}
\end{algorithm}
As in the first theoretical work~\cite{LiZZZ16}, we only set $T=1$ and only consider generating the offspring by applying mutation to the $i$-th individual in the population for the $i$-th subproblem.
As in~\cite{LiZZZ16}, we expect that an optimal (minimal) value of a subproblem corresponds to a Pareto front point of $\omm_k$. Since the Pareto front size is $k+1$, we then set $H=k$ and consider the weights $w_{i}=i/k, i=0,\dots,H$. Note that {in each iteration, $z^*$ is updated before the fitness evaluation of the newly generated offspring. Hence, we easily know that before $1^{n-k}0^k$ ($1^n$) is reached, the minimization of the $0$-th ($H$-th) subproblem is identical to the maximization of $f^k_2(x)$ ($f^k_1(x)$) via the \oea. We also easily know that after both $1^{n-k}0^k$ and $1^n$ are reached,} the $i$-th subproblem $h_i(x)$ is formulated as
\begin{equation}
\begin{split}
h_i(x){}&{}=\max\{w_i|f_1^k(x)-n|,(1-w_i)|f_2^k(x)-n|\}\\
={}&{}\max\Bigg\{\frac{i}{k}\left(n-\sum_{j=1}^n x_j\right), 
\frac{k-i}{k}\left(n-\sum_{j=1}^{n-k}x_j-\sum_{j=n-k+1}^n (1-x_j)\right)\Bigg\}.
\end{split}
\label{eq:hi}
\end{equation}

\subsubsection{Runtime Analysis}
Now we analyze the runtime of the MOEA/D to cover the full Pareto front. We first have the following lemma for the set of the optimal solutions for minimizing $h_i(x)$.
\begin{lemma}
Let {$H=k$ and} ${S_i=\{x\mid f^k(x)=(n-k+i,n-i)\}}$ for $i=1,\dots,H-1$. Then for $h_i(x)$ defined in (\ref{eq:hi}), let $h_i^1(x)=\frac{i}{k}\left(n-f_1^k(x)\right)=\frac{i}{k}\left(n-\sum_{j=1}^nx_j\right)$ 
and $h_i^2(x)=\frac{k-i}{k}\left(n-f_2^k(x)\right)=\frac{k-i}{k}\left(n-\sum_{j=1}^{n-k}x_j-\sum_{j=n-k+1}^n (1-x_j)\right)$, then we have
\begin{itemize}
    \item if $(n-${$|x|_1$} $ > k-i,|x_{[1..n-k]}|_1=n-k)$ or $x\in S_i$, then $h_i(x)=h_i^1(x)$;
   \item if $n-${$|x|_1$} $ \le k-i$, then $h_i(x)=h_i^2(x)$;
   \item $\arg\min_{x\in\{0,1\}^n} h_i(x)=S_i$.
\end{itemize}
\label{lem:optsub}
\end{lemma}
\begin{proof}
Consider any $x\in\{0,1\}^n$. 

We first consider the case when $n-|x|_1 = k-i$, that is, when $f_1^k(x)=n-k+i$. In this case, $h_i^1(x)=(i(k-i))/k$. If $x\in S_i$, that is, 
\begin{align*}
f_2^k(x)= \left(\sum_{j=1}^{n-k} x_j\right) + \left(\sum_{j=n-k+1}^n 1- x_j \right) = n-i,
\end{align*}
then $h_i^2(x)=\frac{k-i}{k}(n-(n-i))=(i(k-${$i$}$))/k$. Hence,
\begin{align}
h_i(x)=h_i^1(x)=h_i^2(x)=\frac{i(k-i)}{k}.
\label{eq:case1}
\end{align}
If $x\notin S_i$, it is not difficult to see that {$|x_{[1..n-k]}|_1<n-k$ and thus $|x_{[n-k+1..n]}|_1 > i$ from $f^k_1(x)=n-k+i$. Then 
$$f^k_2(x) = |x_{[1..n-k]}|_1 +k-|x_{[n-k+1..n]}|_1< n-k+k-i=n-i,
$$}
thus 
\[
h_i^2(x)>\frac{k-i}{k}(n-(n-i))=\frac{i(k-i)}{k}=h_i^1(x).
\]
Hence,
\begin{align}
h_i(x)=h_i^2(x)>h_i^1(x)=\frac{i(k-i)}{k}.
\label{eq:case2}
\end{align}

If $n-|x|_1 > k-i$, then 
\begin{align}
h_i(x)\ge \frac{i}{k}\left(n-\sum_{j=1}^n x_j\right) > \frac{i}{k}(k-i)=\frac{i(k-i)}{k}.
\label{eq:case3}
\end{align}
If further $|x_{[1..n-k]}|_1=n-k$, then $|x_{[n-k+1..n]}|_1<i$ from $n-\sum_{j=1}^n x_j > k-i$. Hence, $f_2^k(x)>n-k+k-i=n-i$. Then $h_i^2(x)<\frac{(k-i)i}{k}<h_i^1(x)$, and thus $h_i(x)=h_i^1(x)${, proving the first property}.

If $n-|x|_1 < k-i$, we have $h_i^1(x)<(i(k-i))/k$, and $\sum_{j=1}^n (1-x_j) < k-i$. Then
\begin{align*}
h_i^2(x)&{}={}\frac{k-i}{k} \left(
n-\sum_{j=1}^{n-k}x_j-\sum_{j=n-k+1}^n (1-x_j)\right)\\
&\ge{}\frac{k-i}{k} {\left(n-\sum_{j=1}^{n-k}x_j-\sum_{j=1}^n (1-x_j)\right)}
>{} \frac{k-i}{k} \left( n-\sum_{j=1}^{n-k}x_j-(k-i)\right)\\
&{}\ge{} \frac{k-i}{k}( n-(n-k)-(k-i))=\frac{i(k-i)}{k}.
\end{align*}
Hence,
\begin{align}
  h_i(x)&=h_i^2(x)\ge\frac{i(k-i)}{k}. 
  \label{eq:case4}
\end{align}
{Thus, with (\ref{eq:case1}) and (\ref{eq:case2}), the second property is proved.}

{From (\ref{eq:case1}) to (\ref{eq:case4}) and noting (\ref{eq:case1}) is for $x\in S_i$, the last property  {is proven}. We also note here}
that we do not consider the case {$(n-|x|_1 > k-i,|x_{[1..n-k]}|_1\ne n-k)$} as it is not needed for our proof in Theorem~\ref{thm:moead}.
\end{proof}

We note that the optimal $h_0$ and $h_H$ are $1^{n-k}0^k$ and $1^n$ respectively. {As discussed before, the minimization of $h_0$ ($h_H$) is identical to the maximization of $f^k_2(x)$ ($f^k_1(x)$) via the \oea.} Then after $O(n\log n)$ generations, both points will be found. Then the subproblems to solve will be (\ref{eq:hi}). Lemma~\ref{lem:optsub} shows that the optimal set of each subproblem is corresponding to one Pareto front point and different subproblem is related to different Pareto front points. Thus solving all subproblems will result in a full Pareto front coverage. Since $T=1$, we only need to calculate the runtime for each subproblem, which is the same as the one that RLS is used to solve this subproblem, and then summing them up results in the final runtime. Hence, we have the following theorem.
\begin{theorem}
Let $H=k$ and $w_i=i/k,i=0,\dots,H$. The expected runtime for the MOEA/D applying one-bit mutation to optimize $\omm_k$ is $O(\max\{k,1\}n\ln n)$.
\label{thm:moead}
\end{theorem}
\begin{proof}[Proof {sketch}]
As discussed above, we only need to consider the runtime of the RLS to solve the $i$-th subproblem for $i=1,\dots,H-1$. We divide the process into two phases. The first phase ends when an $x$ with $|x_{[1..n-k]}|_1=n-k$ is reached for the first time, and the second phase starts right after the end of the first phase, and ends when an optimum of $h_i(x)$ is reached for the first time. 

We note that in the first phase, flipping any zeros in $[1..n-k]$ positions will decrease $h_i(x)$ and flipping any ones in $[1..n-k]$ positions will increase $h_i(x)$. Also note that $[1..n-k]$ bits will not change if the flip happens in $[n-k+1..n]$ positions as we use one-bit mutation. Hence, after at most $O(n\ln (n-k))$ iterations, an $x$ with $|x_{[1..n-k]}|_1=n-k$ will be reached for the first time. In the second phase, for the one-bit mutation, it is not difficult to see that flipping any ones in $[1..n-k]$ positions will increase $h_i(x)$ as well. 
Hence, for the minimization of $h_i(x)$, we know that all individuals in the future generations will keep $|x_{[1..n-k]}|_1=n-k$. Then from Lemma~\ref{lem:optsub}, we know that if the second phase starts with $n-|x|_1>k-i$, then the whole second phase will minimize $h_i^1(x)$, and if it starts with $n-{|x|_1\le~} k-i$, then the whole second phase will minimize $h_i^2(x)$. Similar to the \om optimization, the upper bound for the second phase is $O(n\ln k)$. 

Considering all subproblems proves this theorem.
\end{proof}

\subsection{MOEAs without Decomposition}
The above subsection discussed the runtime for the MOEAs with decomposition. The most widely used  {MOEA} is  {the} \NSGA, which does not rely on decomposing the multiobjective problems into several single-objective problems. In this section we will discuss the runtime for the MOEAs without decomposition.

\subsubsection{Algorithm Descriptions}
(G)SEMO is the basic toy algorithm analyzed in the theory community. It starts with a single randomly generated individual. In each generation, one individual is uniformly at random selected as parent to generate one offspring via one-bit or standard bit-wise mutation. If this offspring cannot be dominated by any individual in the current population, it will kick out all individuals that are weakly dominated by it, and enter into the next generation. See details in Algorithm~\ref{alg:semo}.
\begin{algorithm}[tb]
    \caption{(G)SEMO and SMS-EMOA with population size $\mu$ for  {maximizing} the multiobjective function $f$}
    \begin{algorithmic}[1]
    \STATE{For (G)SEMO: Generate $x \in \{0, 1\}^n$ uniformly at random and $P\leftarrow \{x\}$; For SMS-EMOA: Generate $\mu$ individuals uniformly at random from $\{0,1\}^n$ and form $P$.}
    \FOR {$g=1,2,\dots$}
    \STATE {Select $x$ from $P$ uniformly at random}
    \STATE {Generate $x'$ via applying mutation operator (one-bit mutation for SEMO and standard bit-wise mutation for the GSEMO and SMS-EMOA) on $x$}
    \STATE {$P \leftarrow \SurSel(P,x')$} 
    \ENDFOR
    \end{algorithmic}
    \label{alg:semo}
\end{algorithm}

\begin{algorithm}[tb]
    \caption{$\SurSel(P,x')$}
    \begin{algorithmic}[1]
    \IF {(G)SEMO}
    \IF {there exists no $y\in P$ such that $x' \prec y$}
    \STATE {$P \leftarrow \{z\in P \mid z \npreceq x'\} \cup \{x'\}$}
    \ELSE
    \STATE {$P\leftarrow P$}
    \ENDIF
    \ELSIF {SMS-EMOA}
    \STATE {Use fast-non-dominated-sort() in~\cite{DebPAM02} to divide $P\cup\{x'\}$ into fronts
      $F_1,F_2,…,F_d$ for some positive integer $d$}
     \STATE {{Uniformly at random choose $y$ from} $\arg\min_{z\in F_d} \HV(F_d)-\HV(F_d\setminus \{z\})$}
     \STATE {$P\leftarrow F_d\setminus \{y\}$}
    \ENDIF
    \end{algorithmic}
    \label{alg:survival}
\end{algorithm}

Different from (G)SEMO, the \NSGA~\cite{DebPAM02} works with a fixed size $N$ of population. In each generation, $N$ offspring individuals are generated, and it uses non-dominated sorting and crowding distance to remove the worst $N$ individuals in the combined parent and offspring population $R_t$. The non-dominated sorting procedure will first divide $R_t$ into several front $F_i$, with $F_i$ the non-dominated solutions in $R_t\setminus \cup_{j=0}^{i-1}F_j$. For the critical front $F_{i^*}$ with $i^*=\min\{j\mid |\cup_{i=0}^{j}${$F_i$}$|\ge N\}$, the algorithm calculates the crowding distances of all individuals in $F_{i^*}$ and keep the ones with $N-|\cup_{j=0}^{i^*-1}F_j|$ largest crowding distance (tie broken at random). See details in Algorithm~\ref{alg:nsga}. It is the most widely used MOEAs and the first runtime is conducted in~\cite{ZhengLD22,ZhengD23aij}. A series of {theoretical} analyses on the \NSGA are conducted later, like~\cite{ZhengD25approx,ZhengD24tec,BianQ22,DoerrQ23tec,DoerrQ23LB,CerfDHKW23,DangOSS23gecco,DangOS24ppsn,DengZLLD24,BianRLQ24ijcai,LuBQ24}.
\begin{algorithm}[!ht]
    \caption{\mbox{NSGA-II} with population size $N$ for maximizing the multiobjective function $f$}
    \begin{algorithmic}[1]
    \STATE {Uniformly at random generate the initial population $P_0=\{x_1,x_2,\dots,x_N\}$ with $x_i\in\{0,1\}^n,i=1,2,\dots,N.$}\label{ste:initialize}
    \FOR{$t = 0, 1, 2, \dots$} \label{ste:iterate}
    \STATE {Generate the offspring population $Q_t$ with size $N$}\label{ste:generate}
    \STATE {Use fast-non-dominated-sort() in~\cite{DebPAM02} to divide $R_t=P_t\cup Q_t$ into $F_1,F_2,\dots$}
    \label{ste:sort}
    \STATE {Find $i^* \ge 1$ such that $\sum_{i=1}^{i^*-1}|F_i| < N$ and $\sum_{i=1}^{i^*}|F_i| \ge N$}\label{ste:rank}
    \STATE {Use Algorithm~\ref{alg:cDis} to separately calculate the crowding distance of each individual in $F_{1},\dots,F_{i^*}$}\label{ste:cDis}
    \STATE {Let $\tilde{F}_{i^{*}}$ be the $N-\sum_{i=0}^{i^*-1}|F_{i}|$ individuals in $F_{i^*}$ with largest crowding distance, chosen at random in case of a tie}\label{ste:final front}
    \STATE {$P_{t+1}=\left(\bigcup_{i=1}^{i^*-1}F_i\right)\cup\tilde{F}_{i^*}$}\label{ste:new parents}
    \ENDFOR 
    \end{algorithmic}
    \label{alg:nsga}
\end{algorithm}
\begin{algorithm}[tb]
    \caption{crowding-distance($S$)}
     \textbf{Input:} $S=\{S_1,\dots,S_{|S|}\}$: the set of individuals\\
     \textbf{Output:} $\cDis(S)=(\cDis(S_1),\dots,\cDis(S_{|S|}))$, the vector of crowding distances of the individuals in $S$
		
    \begin{algorithmic}[1]
    \STATE $\cDis(S)=(0,\dots,0)$
    \FOR {each objective function $f_i$}
    \STATE {Sort $S$ in order of ascending $f_i$ value: $S_{i.1},\dots,S_{i.{|S|}}$}
    \STATE {$\cDis(S_{i.1})=+\infty, \cDis(S_{i.{|S|}})=+\infty$}
    \FOR {$j=2,\dots, |S|-1$}
    \STATE {$\cDis(S_{i.j})=\cDis(S_{i.j}) + \frac{f_i(S_{i.{j+1}})-f_i(S_{i.{j-1}})}{f_i(S_{i.{|S|}})-f_i(S_{i.1})}$}
    \ENDFOR
    \ENDFOR
    \end{algorithmic}
    \label{alg:cDis}
\end{algorithm}

The \SMS~\cite{BeumeNE07} is a variant of the steady-state {(in a $(\mu+1)$ scheme)} \NSGA, which replaces the crowding distance by the hypervolume contribution where the hypervolume is defined in the following.  {Note that in the rest of the paper, one could set the reference point $r$ strictly less than all possible function values, like $r=(-1,-1)$, and we will not explicitly state its value in our theoretical results.}
\begin{definition}[Hypervolume]
    {Let a reference point $r$ be dominated by all Pareto optimal solutions.} The \emph{hypervolume} of a set $S$ of individuals w.r.t. $r$ in the objective space is defined as
\begin{align*}
\HV_r(S)= \mathcal{L}\left(\,\bigcup_{u\in S} \{h\in\R^m \mid r \le h \le f(u)\}\right),
\end{align*}
where $\mathcal{L}$ is the Lebesgue measure{, and for $u,v\in\R^m$, $u\le v$ is defined by $u_i\le v_i$ for all $i\in[1..m]$}. 
\end{definition}
The hypervolume contribution of an individual is the hypervolume difference between the cases of including this individual or not. The individual with the smallest hypervolume contribution will be removed. See details in Algorithm~\ref{alg:semo}. 
The first runtime analysis of the \SMS is conducted by Bian et al.~\cite{BianZLQ25} for the bi-objective \ojzj, and a stochastic population update strategy is proposed. Later, its runtime on \dltb is conducted by Zheng et al.~\cite{ZhengLDD24}. Zheng and Doerr~\cite{ZhengD24} proved the first runtime of the \SMS for many objectives, and showed that it performs well compared with the inefficiency of the \NSGA. Besides, they also proved that the speed-up witnessed by the stochastic population in~\cite{BianZLQ25} for two objectives has a reduced impact when the number of objectives grows, but proved that the heavy-tailed mutation has a good runtime speed-up.

\subsubsection{Runtime Analyses}
Now we discuss the runtime of the (G)SEMO, the \NSGA, and the \SMS. As mentioned before, Antipov, K\"otzing, and Radhakrishnan~\cite{AntipovKR24} proved the expected runtime of $O(kn\ln n)$ for (G)SEMO for an essentially identical problem of minimizing $(\OO_a,\OO_b)$ with $k=H(a,b)$. Here, for completeness, we also include the runtime of the (G)SEMO in our analysis as its similar proof procedure to that for the \NSGA and \SMS. The following lemma estimates the probability of selecting an individual as a parent {in one iteration for any algorithm that generates the offspring population with the same size as the parent population $P$. We consider the following three ways. \emph{The fair selection} will let each individual in $P$ be selected as a parent once. In the \emph{random selection}, we $|P|$ times independently and uniformly at random pick a parent from $P$. The \emph{binary tournament selection} generates $|P|$ parents by independently conducting $|P|$-round competitions (with replacement). In each round, two individuals in $|P|$ are uniformly at random selected, and the better (e.g. based on the non-dominated sorting and crowding distance for the \NSGA) one will be picked as a parent.}
\begin{lemma}
For a population $P$ with size $N$, we have the following results.
\begin{itemize}
\item For any $x\in P$, the probability of having $x$ selected as a parent {in an iteration} is $1$ for fair selection, and at least $1-1/e$ for random selection.
\item Let $m,n\in \N_{+}$ and $m=o(N)$. Let $F'_1$ be the set of non-dominated solutions {(within $P$)} w.r.t. some $m$-objective $f$. Let $f_1^{\max}=\max\{f_1(x) \mid x\in F'_1\}$, $f_i^{\max}=\max\{f_i(x) \mid x\in F'_1, f_j(x)=f_j^{\max},j=1,\dots,i-1\}$ for $i=2,\dots,m$, and $X=\{x\in F'_1 \mid f(x)=(f_1^{\max}, \dots, f_m^{\max})\}$. Then {for the \NSGA} there is an individual $\tilde{x}\in X$ with $\cDis(\tilde{x})=+\infty$, and the probability that $\tilde{x}$ is selected as a parent {in an iteration} is $\Omega(1)$ for binary tournament selection. 
\end{itemize}
\label{lem:parent}
\end{lemma}
\begin{proof}
Since each individual in $P$ will be selected as a parent for fair selection, the probability of selecting $x$ as a parent is $1$. For random selection, such probability is
\begin{align*}
1-\left(\frac{N-1}{N}\right)^N \ge 1-\frac 1e.
\end{align*}

Now we consider the binary tournament selection. Since $X \subset F'_1$ and $X$ contains individuals with the largest $f_1$ value, then there is an $\tilde{x}\in X$ with $\cDis(\tilde{x})=+\infty$ from Algorithm~\ref{alg:cDis}. Hence, $\tilde{x}$ will win in the competition against other individuals with a finite crowding distance. Note that there are at most $2m$ individuals with infinite crowding distances. Hence, in each round of comparison, the probability of choosing $\tilde{x}$ and one with finite crowding distance is at least
\begin{align*}
\frac{{\binom{N-2m}{1}}}{{\binom{N}{2}}}=\frac{2(N-2m)}{N(N-1)}.
\end{align*} 
Then the probability of selecting $\tilde{x}$ as a parent is at least
\begin{equation*}
\begin{split}
1-&{}\left(1-\frac{2(N-2m)}{N(N-1)}\right)^N {= 1-\left(1-\frac{2(N-2m)}{N(N-1)}\right)^{\frac{N(N-1)}{2(N-2m)}\frac{2(N-2m)}{N-1}}} \\
&{}\ge 1-\exp\left(-\frac{2(N-2m)}{N-1}\right).
\qedhere
\end{split}
\end{equation*}
\end{proof}

The following lemma shows that once a Pareto front point is reached, all future generations {will contain at least one individual with its function value as such a Pareto front. The result for the (G)SEMO is trivial, the one for the \SMS is a corollary of~\cite[Lemma~4]{ZhengD24}, and the result for the \NSGA can be easily obtained from~\cite[Proof of Lemma~1]{ZhengD23aij} that each function value will have at most four corresponding individuals with positive crowding distance values.}
\begin{lemma}[\cite{ZhengD24,ZhengD23aij}]
Consider using (G)SEMO, SMS-EMOA with $\mu \ge k+1$, or the NSGA-II with $N\ge 4(k+1)$ to optimize {the} $\omm_k$ problem. {At some time $t_0$, let $x$ be a solution that exists in the current population $P_{t_0}$ or is one offspring of $P_{t_0}$. Then for any $t>t_0$ there is a $y$ in the population $P_t$ such that $y\succeq x$. Particularly, if $x$ is Pareto optimal, then there is a $y\in P_t$ for any $t>t_0$ such that $f^k(y)=f^k(x)$.}
\label{lem:sur}
\end{lemma}
The following lemma estimates the runtime to reach at least one Pareto optimal solution. {Note here that in the following, all theoretical results hold for the \NSGA with standard bit-wise mutation or one-bit mutation.}
\begin{lemma}
Consider using (G)SEMO, SMS-EMOA with $\mu \ge k+1$, or the NSGA-II with $N\ge 4(k+1)$ to optimize {the} $\omm_k$ problem. After $O(\max\{k,1\}n\ln n)$ ((G)SEMO), $O(\mu n\ln n)$ (SMS-EMOA), and $O(nN\ln n)$ (NSGA-II) number of function evaluations in expectation, the population (and all populations {afterwards}) will contain at least one Pareto optimal solution. 
\label{lem:firstP}
\end{lemma}
\begin{proof}
We consider the runtime to reach a Pareto optimal solution for the first time. More specifically, we consider the first time to reach $1^n$. Now we reuse notations in Lemma~\ref{lem:parent}. For the current population $P$, let $\eta=\max_{x\in P} f_1(x)$. Then it is not difficult to see that $\eta = f_1^{\max}$ and $\eta$ will not decrease as the population evolves. Let $\tilde{x}\in X$ with infinite crowding distance. Note that the probability of selecting $\tilde{x}$ as a parent {in an iteration} is $\Omega(1)$ for the NSGA-II from Lemma~\ref{lem:parent}, {$\frac{1}{|P|} \ge \frac{1}{k+1}$} for the (G)SEMO from Lemma~\ref{lem:popsize}, and $\frac{1}{\mu}$ for the SMS-EMOA. 

We note that the probability to generate an offspring $x'$ with $f_1(x') > \eta$ from {the chosen} $\tilde{x}$ is at least
\begin{align*}
\frac{n-\eta}{n}\left(1-\frac1n\right)^{n-1}\ge \frac{n-\eta}{en}
\end{align*}
for the standard bitwise mutation, and it is $\frac{n-\eta}{n}$ for the one-bit mutation. 

Now we show that if the offspring population $Q$ contains an individual with $f_1$ value larger than $\eta$, then at least one individual $y$ with $f_1(y) > \eta$ will survive to the next population for all algorithms. For (G)SEMO, $Q=\{y\}$ and $y$ is not dominated by others, and thus $y$ will survive. For SMS-EMOA, $Q=\{y\}$ and $y$ is not dominated by others as well. Then $y\in F_1$ and its hypervolume contribution is positive. Note that the only removal will happen in the last front $F_d$ for $d>1$, or in $F_1$ with minimal hypervolume contribution for $d=1$. For the former case, $y$ will enter into the next generation as $y\in F_1$. For the latter case, since $|f_1(F_1)|\le k+1$ from Lemma~\ref{lem:popsize} and $|P\cup Q| \ge N+1>k+1$, there is an individual with hypervolume contribution as zero. One individual with zero hypervolume contribution will be removed and $y$ will survive. For the NSGA-II, since $Q$ contains an individual with $f_1$ value larger than $\eta$, we know that there is an individual $y$ with the largest $f_1$ value and with infinite crowding distance. Obviously, $y\in F_1$. Note that each function value will have at most four corresponding individuals with positive crowding distance values. Since $|f(F_1)| \le k+1$ from Lemma~\ref{lem:popsize}, we know that there are at most $4(k+1)$ individuals with positive crowding distance. With $N\ge 4(k+1)$, we know that $y$ will survive to the next population.

Since $\eta$ does not decrease, we know that the expected iterations to increase $\eta$ is at most $O(\frac{n}{n-\eta})$ iterations for the NSGA-II, at most $O\left(\frac{(k+1)n}{n-\eta}\right)$ for the (G)SEMO, and at most $O\left(\frac{\mu n}{n-\eta}\right)$ for the SMS-EMOA. Hence, to reach $1^n$, that is, $\eta=n$, we have the upper bound of expected iterations of
\begin{align*}
\sum_{\eta=0}^{n-1}O\left(\frac{n}{n-\eta}\right) = O(n\ln n)
\end{align*}
for the NSGA-II,
\begin{align*}
\sum_{\eta=0}^{n-1}O\left(\frac{(k+1)n}{n-\eta}\right) = O(n\max\{k,1\}\ln n)
\end{align*}
for the (G)SEMO, and 
\begin{align*}
\sum_{\eta=0}^{n-1}O\left(\frac{\mu n}{n-\eta}\right) = O(\mu n\ln n)
\end{align*}
for the SMS-EMOA. {Noting that $N$ offspring solutions are generated and evaluated in each iteration for the \NSGA, and one offspring solution for the (G)SEMO and \SMS, we proved this lemma.}
\end{proof}

The following lemma shows the runtime for the full coverage of the Pareto front after one Pareto front point is reached.
\begin{lemma}
Consider using (G)SEMO, SMS-EMOA with $\mu \ge k+1$, or the NSGA-II with $N\ge 4(k+1)$ to optimize {the} $\omm_k$ problem. Assume that the population contains a Pareto optimal solution. Then after $O(\max\{k,1\}n\max\{\ln k,1\})$ ((G)SEMO), $O(\mu n\max\{\ln k,1\})$ (SMS-EMOA), and $O(nN\max\{\ln k,1\})$ (NSGA-II) number of function evaluations in expectation, the population (and all populations afterward) will cover the whole Pareto front. 
\label{lem:fullP}
\end{lemma}
\begin{proof}
{For all $i\in[0..k-1]$, w}e call $(n-k+i,n-i)$, and $(n-k+i+1,n-i-1)$ are \emph{neighbors}, and also $(n-k+i,n-i)$ is the \emph{left neighbor} of $(n-k+i+1,n-i-1)$ and $(n-k+i+1,n-i-1)$  is the \emph{right neighbor} of $(n-k+i,n-i)$. Let $x$ be a Pareto optimum in the current population, satisfying that at least one neighbor of $f(x)=(n-k+j,n-j)$ is not reached. {Then we know that $x$ contains $j$ ones and $k-j$ zeros in $x_{[n-k+1...n]}$. Hence, it is sufficient to flip exactly a one bit in $x_{[n-k+1...n]}$ (with other bits in $x$ unchanged) to create a left neighbor, and exactly a zero bit in $x_{[n-k+1...n]}$ for a right neighbor. Therefore,} if the left neighbor of $f(x)$ is not reached, then the probability to generate $x'$ with $f(x')=(n-k+j-1,n-j+1)$ from $x$ is $j/n$ for the one-bit mutation, and at least
\begin{align*}
\frac{j}{n}\left(1-\frac 1n\right)^{n-1} \ge \frac{j}{en}
\end{align*}
for the standard bitwise mutation. That is, such probability is $\Omega\left(\frac{j}{n}\right)$ for both mutation strategies. Similarly, if the right neighbor of $f(x)$ is not reached, then the probability to generate $x'$ with $f(x')=(n-k+j+1,n-j-1)$ from $x$ is $\tfrac{k-j}{n}$
for the one-bit mutation, and at least 
\begin{align*}
\frac{k-j}{n}\left(1-\frac 1n\right)^{n-1} \ge \frac{k-j}{en}
\end{align*}
for the bitwise mutation. That is, such probability is $\Omega\left(\frac{k-j}{n}\right)$ for both mutation strategies.

Together with the probability {that} $x$ {is chosen} as a parent in one iteration from Lemma~\ref{lem:parent}, we know that for the NSGA-II, the probability is $\Omega\left(\frac{j}{n}\right)$ to generate a left neighbor of $x$ in one iteration and is $\Omega\left(\frac{k-j}{n}\right)$ to generate a right neighbor. For the (G)SEMO, the corresponding two probabilities are $\Omega\left(\frac{j}{n|P|}\right) =\Omega\left(\frac{j}{n(k+1)}\right)$ and $\Omega\left(\frac{k-j}{n|P|}\right)=\Omega\left(\frac{k-j}{n(k+1)}\right)$. For the SMS-EMOA, the corresponding two probabilities are $\Omega\left(\frac{j}{n\mu}\right)$ and $\Omega\left(\frac{k-j}{n\mu}\right)$.

From Lemma~\ref{lem:sur}, we know that any reached Pareto front point will be maintained in all future populations. Hence, the expected number of iterations to cover the full Pareto front is 
\begin{align*}
\sum_{j=1}^{k} O\left(\frac{n(k+1)}{j}\right) + \sum_{j=0}^{k-1} O\left(\frac{n(k+1)}{k-j}\right) =O(n(k+1)\max\{\ln k,1\})
\end{align*}
for the (G)SEMO,
\begin{align*}
\sum_{j=1}^{k} O\left(\frac{n\mu}{j}\right) + \sum_{j=0}^{k-1} O\left(\frac{n\mu}{k-j}\right) =O(n\mu\max\{\ln k,1\})
\end{align*}
for the SMS-EMOA, and
\begin{align*}
\sum_{j=1}^{k} O\left(\frac{n}{j}\right) + \sum_{j=0}^{k-1} O\left(\frac{n}{k-j}\right) =O(n\max\{\ln k,1\})
\end{align*}
for the NSGA-II.
\end{proof}

Together with the above runtime results, we have the following runtime for the full coverage of the Pareto front.
\begin{theorem}
Consider using (G)SEMO, SMS-EMOA with $\mu \ge k+1$, or the NSGA-II with  $N\ge 4(k+1)$ to optimize $\omm_k$ problem. Then the expected number of function evaluations to cover the full Pareto front is $O(\max\{k,1\}n\ln n)$ for (G)SEMO, $O(n\mu\ln n)$ for the SMS-EMOA, and $O(nN\ln n)$ for the NSGA-II.
\label{thm:moeas}
\end{theorem}

\subsection{Summary and Comment}
Recall in Section~\ref{sec:eps} that with careful settings of $\eps$ and penalty coefficient $r$ in the exterior penalty function way, the $\eps$-constraint approach can cover the full Pareto front in expected $O(\max\{k,1\}n\ln n)$ function evaluations. Theorems~\ref{thm:moead} and~\ref{thm:moeas} show that the generally analyzed MOEAs, the MOEA/D, the (G)SEMO, the \NSGA, and the \SMS can easily cover the full Pareto front with the same {asymptotic} runtime. 

Besides, as mentioned before, if we merely consider solving the constrained problem (\ref{eq:con}) as our goal, Theorem~\ref{thm:rls} in Section~\ref{sec:eps} shows that although the exterior penalty function way can solve it, a careful design of the penalty coefficient is required. Theorems~\ref{thm:moead} and~\ref{thm:moeas} show the possibility of solving a constrained problem via MOEAs. That is, transfer a constrained problem to a multiobjective optimization problem, and then pick the best solution w.r.t. the original constrained problem among the Pareto set. This is an easy way to implement with the cost of more runtime, like in this example, $O(\max\{k,1\}n\ln n)$ (Theorems~\ref{thm:moead} and~\ref{thm:moeas}) for the MOEAs but $O(n\ln n)$ (Theorem~\ref{thm:rls}) for the RLS to solve the exterior penalty function with carefully chosen coefficient.


\section{{Brief Discussion of a \lo Variant}\label{sec:lotz}}
Our comparison of MOEAs and non-MOEA methods was built on a bi-objective \om variant. To address the question of whether these observations hold for other problems, this section will give brief discussions on performance for the variant of \lo, another popular benchmark analyzed in the evolutionary theory community. We will show that similar situations are witnessed for the non-MOEAs and MOEAs. As a side result, when there is no conflict in the two objectives, the discussed MOEAs reach the same runtime of $O(n^2)$ as many single-objective algorithms for the maximization of $\lo$. 

\subsection{$\lotz_k$}
As discussed in Section~\ref{ssec:bench}, the \lotz benchmark is a popular bi-objective counterpart of the single-objective \lo benchmark. Indeed, we note that this benchmark can be regarded as a special case of the following benchmark  {with $k=n$}.
\begin{definition}
\label{def:lotzk}
    Let $\ell\in \N$, and let $\LO(y)=\sum_{i=1}^{\ell}\prod_{j=1}^iy_j$ and $\TZ(y)=\sum_{i=1}^{\ell}\prod_{j=n-i+1}^{\ell} (1-y_j)$ for any $y=(y_1,\dots,y_{\ell})$. The $\lotz_k$ function $\{0,1\}^n\rightarrow \R^2$ is defined by 
    \begin{align*}
        h^k(x)=(h_1^k(x),h_2^k(x))=(\LO(x),\LO(x_{[1..n-k]})+\TZ(x_{[n-k+1..n]}))
    \end{align*}
    for $x=(x_1,\dots,x_n)\in\{0,1\}^n$.
\end{definition}
Intuitively,  {the} two objectives of $\lotz_k$ concur in the first $n-k$ bit positions but conflict in the last $k$ bit positions. We call $k$ the degree of conflict of this benchmark. The following two lemmas respectively give the Pareto front and the maximal number of non-dominated function values (the maximal number of incomparable solutions).
\begin{lemma}
    Let $M$ denote the Pareto front of $\lotz_k$. Then 
    \begin{align*}
        M=\{(n-k+i,n-i) \mid i\in[0..k]\},
    \end{align*}
     {and} the Pareto set  {for $\lotz_k$ is} $S^*:=\{1^{n-k+i}0^{k-i} \mid i\in[0..k]\}$.
    \label{lem:PFlo}
\end{lemma}
\begin{proof}
We first show that $y\not\succ x$ for any $x\in S^*$ and for any $y\in\{0,1\}^n$. If $y\succ x$, from the definition of dominance, we know 
\begin{align}
h^k_1(y)+h^k_2(y) > h^k_1(x)+h^k_2(x).
\label{eq:yx}
\end{align}
From Definition~\ref{def:lotzk} we have
\begin{align*}
    h^k_1(y)+h^k_2(y)&{}={}2\LO(y_{[1..n-k]})+\sum_{i=n-k+1}^n\prod_{j=1}^i y_j + \TZ(y_{[n-k+1..n]})\\
    &{}\le{} 2(n-k)+k=2n-k=h^k_1(x)+h^k_2(x),
\end{align*} 
which contradicts (\ref{eq:yx}).  {Hence} $y\not\succ x$.

It remains to show that for any $y\in\{0,1\}^n\setminus S^*$, there is an $x\in S^*$ such that $x\succ y$. Let $a=\LO(y)${. If $a< n-k$, then let $x=1^{n-k}0^k$. Then $x\in S^*$, $h_1^k(y)< n-k=h_1^k(x)$, and $h_2^k(y)<n=h_2^k(x)$. Thus $x\succ y$. If $a\ge n-k$, then} let $x=1^a0^{n-a}$. Then $a<n$ and $x\in S^*$. Since $y\notin S^*$, we know  {$h_2^k(y)<n-k+n-a=h_2^k(x)$}. Noting that $h_1^k(y)=a=h_1^k(x)$, we know that $x\succ y$. 
\end{proof}

\begin{lemma}
    The maximal number of pairwise non-dominated function values w.r.t. $\lotz_k$ is $k+1$.
    \label{lem:popsizelo}
\end{lemma}
\begin{proof}
    Let $V$ be the set of mutually incomparable solutions. We  {now} show that for any incomparable $x,y\in V$,  {we have} $\TZ(x_{[n-k+1..n]})\neq \TZ(y_{[n-k+1..n]})$. If $\TZ(x_{[n-k+1..n]})= \TZ(y_{[n-k+1..n]})$, from Definition~\ref{def:lotzk} we have that
    \begin{itemize}
        \item if $\LO(x_{[1..n-k]})>\LO(y_{[1..n-k]})$, then $h^k(x)>h^k(y)$;
        \item if $\LO(x_{[1..n-k]})<\LO(y_{[1..n-k]})$, then $h^k(x)<h^k(y)$;
        \item if $\LO(x_{[1..n-k]})=\LO(y_{[1..n-k]})$, then $h^k_2(x)=h^k_2(y)$,
    \end{itemize}
    any of which contradicts that $x$ and $y$ are incomparable. Since $\TZ(x_{[n-k+1..n]})\in[0..k]$, we have $|V|\le k+1$. With the Pareto front size of $k+1$ from Lemma~\ref{lem:PFlo}, this lemma is proven.
\end{proof}
\subsection{Difficulty of the Scalarization Approach}
In this subsection, we will discuss the performance of the scalarization approach to maximize $\lotz_k$. It will transfer the original problem into the maximization of the following function.
\begin{definition}
    Let $w\in\R$. The scalarization function $h_w^k:\{0,1\}^n\rightarrow \R$ of $\lotz_k$ is defined by
    \begin{align*}
        h_w^k(x)&{}={}wh_1^k(x)+(1-w)h_2^k(x)\\
        &{}={}\LO(x_{[1..n-k]})+w\sum_{i=n-k+1}^n\prod_{j=1}^i x_j+(1-w)\TZ(x_{[n-k+1..n]}),
    \end{align*}
    for $x=(x_1,\dots,x_n)\in\{0,1\}^n$.
\end{definition}

The following lemma calculates the optimal solution set w.r.t. the maximization of $h_w^k$.
\begin{lemma}
    Let $S^*$ be defined as in Lemma~\ref{lem:PFlo}. Let $S_w=\{x\in\{0,1\}^n \mid h_w^k(y) \le h_w^k(x), \text{ for all }y \in \{0,1\}^n\}$ be the set of optima of $h_w^k$. Then $S_w=\{1^{n-k}0^k\}$ for $w\in[0,1/2)$, $S_w=\{1^{n}\}$ for $w\in(1/2,1]$, and $S_w=S^*$ for $w=1/2$.
    \label{lem:wsoptlo}
\end{lemma}
\begin{proof}
    If $w\in[0,1/2)$, for any $y\ne 1^{n-k}0^k$, then either $\LO(y_{[1..n-k]})\neq n-k$ or $\TZ(y_{[n-k+1..n]})\ne k$. For $\LO(y_{[1..n-k]})\neq n-k$, we have $\LO(y_{[1..n-k]})< n-k$. With $\TZ(y_{[n-k+1..n]})\le k$ and $w<1/2$, we know
    \begin{equation}
    \begin{split}
        h_w^k(y)&{}={}\LO(y_{[1..n-k]})+w\sum_{i=n-k+1}^n\prod_{j=1}^i y_j+(1-w)\TZ(y_{[n-k+1..n]})\\
        &{}={}\LO(y_{[1..n-k]})+(1-w)\TZ(y_{[n-k+1..n]})\\
        &{}<{}n-k+(1-w)k=h_w^k(1^{n-k}0^k).
        \end{split}
        \label{eq:hlo}
    \end{equation}
    For $\TZ(y_{[n-k+1..n]})\ne k$, we have $\TZ(y_{[n-k+1..n]})<k$. With $\LO(y_{[1..n-k]})\le n-k$ and $w<1/2$, we know
    \begin{align*}
        h_w^k(y)&{}={}\LO(y_{[1..n-k]})+w\sum_{i=n-k+1}^n\prod_{j=1}^i y_j+(1-w)\TZ(y_{[n-k+1..n]})\\
        &{}\le{} \LO(y_{[1..n-k]})+w(k-\TZ(y_{[n-k+1..n]}))+(1-w)\TZ(y_{[n-k+1..n]})\\
        &{}={}\LO(y_{[1..n-k]})+wk+(1-2w)\TZ(y_{[n-k+1..n]})\\
        &{}<{}n-k+wk+(1-2w)k=n-wk=h_w^k(1^{n-k}0^k).
    \end{align*}

    If $w\in(1/2,1]$, for any $y\ne 1^n$, if further $\LO(y_{[1..n-k]})\ne n-k$, noting that (\ref{eq:hlo}) holds for all $w\in[0,1]$, we have
    \begin{align*}
        h_w^k(y)<n-k+(1-w)k<n-k+wk=h_w^k(1^n),
    \end{align*}
    where the last inequality uses $1-w<w$ for $w>1/2$. 
    For the case when $\LO(y_{[1..n-k]})=n-k$, we have $\sum_{i=n-k+1}^n\prod_{j=1}^k y_j<k$, and thus
    \begin{align*}
        h_w^k(y)&{}={}\LO(y_{[1..n-k]})+w\sum_{i=n-k+1}^n\prod_{j=1}^i y_j+(1-w)\TZ(y_{[n-k+1..n]})\\
        &{}\le{} \LO(y_{[1..n-k]})+w\sum_{i=n-k+1}^n\prod_{j=1}^i y_j+(1-w)\left(k-\sum_{i=n-k+1}^n\prod_{j=1}^i y_j\right)\\
        &{}={}n-k+(1-w)k+(2w-1)\sum_{i=n-k+1}^n\prod_{j=1}^i y_j\\
        &{}<{}n-k+(1-w)k+(2w-1)k=n-k+wk=h_w^k(1^{n}).
    \end{align*}

    If $w=1/2$, for any $y\notin S^*$, let $a=\LO(y)$, then $\LO(y_{[1..n-k]})\ne n-k$ if $a< n-k$. Noting that (\ref{eq:hlo}) holds for all $w<1$ including the case of $w=1/2$, we know that $h^k_w(y)<h_w^k(1^{n-k}0^k)$. If $a\ge n-k$, let $\tilde{x}=1^a0^{n-a}$. Then we know that $\tilde{x}\in  {S^*}$, $h^k_1(y)=h^k_1(\tilde{x})$, and $h^k_2(y)<h^k_2(\tilde{x})$. Thus 
    \begin{align*}
        h^k_w(y)=wh^k_1(y)+(1-w)h^k_2(y)<wh^k_1(\tilde{x})+(1-w)h^k_2(\tilde{x})=h^k_w(\tilde{x}).
    \end{align*}
    For any $y\in S^*$, since $\sum_{i=n-k+1}^n\prod_{j=1}^i y_j+\TZ(y_{[n-k+1..n]})=k$, we easily know that $h^k_w({y})=h^k_w(\tilde{x})$.
\end{proof}

From Lemma~\ref{lem:wsoptlo}, we easily obtain the following theorem to show the difficulty of the scalarization approach.
\begin{theorem}
    Let $M$ be defined as in Lemma~\ref{lem:PFlo} and $k>2$. Let $S\subset \R$ be a set of $w$,
    and $x_{w}$ be one global optimum of $h_w^k$. Let $F=\{h^k(x_{w})\mid w\in S\}$. Then $M \not\subseteq F$.
    \label{thm:wslo}
\end{theorem}

\subsection{Inconveniences of $\eps$-Constraint Approach}\label{ssec:epslo}
This subsection discusses the performance of the $\eps$-constraint approach to maximize $\lotz_k$. Similar to Section~\ref{sec:eps}, it reformulates $\lotz_k$ into the following problem.
\begin{equation}
    \begin{split}
        \max h_1^k(x)=\LO(x) \text{ s.t. }\LO(x_{[1..n-k]})+\TZ(x_{[n-k+1..n]})\ge \eps.
    \end{split}
    \label{eq:epslo}
\end{equation}
It is not difficult to see that no feasible solution exists for $\eps>n$. Among all feasible solutions, the optimal solution set is $\{1^n\}$ if $\eps\le n-k$ and $\{1^{n-i}0^i\}$ if $\eps\in(n-k+i-1,n-k+i]$ for some $i\in[1..k]$. Hence, together with the Pareto front for $\lotz_k$ (Lemma~\ref{lem:PFlo}), we know that for a set of constrained problems (\ref{eq:epslo}) with $\eps\in\{\eps_0,\dots,\eps_k\}$ where $\eps_i\in (n-k+i-1,n-k+i]$ for $i\in[1..k]$ and $\eps_0\in(-\infty,n-k]$, the set of optimal solutions results in a full coverage for $\lotz_k$.

\subsubsection{Penalty Ways with Proper Parameters}
The exterior penalty way with coefficient $r>0$ to solve (\ref{eq:epslo}) is to maximize
\begin{equation}
    g(x)=h_1^k(x)+r\min\{0,h_2^k(x)-\eps\}.
    \label{eq:epflo}
\end{equation}
As discussed before, the nonparameter penalty way is identical to the exterior one with $r=1$.

 {Since this section serves as a brief discussion, u}nlike the complete calculation of the optimal solution set for the maximization of (\ref{eq:epf}) in Section~\ref{sssec:expenalty}, here we only {focus on} the optimal solution set for some specific settings of $\eps$ and $r${. This is sufficient} to show the possibility of using a non-MOEA way to properly solve the multi-objective $\lotz_k$ problem  { without distracting the readers}. The following theorem shows the optimal solution set of the maximization of (\ref{eq:epflo}) with different settings of $\eps$ and $r$.

\begin{theorem}
    Let $\eps_i=n-k+i$ for $i\in[0..k]$. Consider the maximization of (\ref{eq:epflo}) with $\eps=\eps_i$ and a given value of $r$. If $r>1$, the optimal solution set is $\{1^{n-i}0^i\}$. If $r=1$, the optimal solution set is $\{1^{n-j}0^j\mid j\in[0..k]\}$.
    \label{thm:optlo}
\end{theorem}
\begin{proof}
    If $x$ is feasible, that is, $h_2^k(x)\ge \eps_i=n-k+i$, then $\TZ(x_{[n-k+1..n]})\ge i$ as $\LO(x_{[1..n-k]})\le n-k$. Hence, $g(x)=h_1^k(x)\le n-\TZ(x_{[n-k+1..n]}) \le n-i$, where $g(x)=n-i$ holds if and only if $\TZ(x_{[n-k+1..n]})= i$ and $\LO(x)=n-i$, that is, if and only if $x=1^{n-i}0^i$. Hence among the feasible solutions, $x=1^{n-i}0^i$ is the unique maximizer of (\ref{eq:epflo}).
    
    Let now $x$ be infeasible, that is, $h_2^k(x)< \eps_i=n-k+i$. If $\LO(x)<n-k$, then we have
    \begin{align*}
        g(x)=h_1^k(x)+r(h_2^k(x)-(n-k+i)) < n-k,
    \end{align*}
    simply because $r>0$.
    If $\LO(x)\ge n-k$, then $\LO(x_{[1..n-k]})=n-k$ (that is, $x_{[1..n-k]}=1^{n-k}$), thus $\LO(x)=n-k+\LO(x_{[n-k+1..n]})$ and $\TZ(x_{[n-k+1..n]})<i$ as $x$ is infeasible. Together with $g(1^{n-i}0^i)=h_1^k(1^{n-i}0^i)=n-i$, for $r\ge 1$ we have
    \begin{align*}
        g(x){}&{}=h_1^k(x)+r(h_2^k(x)-(n-k+i))\\
        ={}&{}n-k+\LO(x_{[n-k+1..n]})+g(1^{n-i}0^i)-(n-i)+r(\TZ(x_{[n-k+1..n]})-i)\\
        ={}&{}g(1^{n-i}0^i)+\LO(x_{[n-k+1..n]})-k+i+r(\TZ(x_{[n-k+1..n]})-i)\\
        \le{}&{} g(1^{n-i}0^i)-\TZ(x_{[n-k+1..n]})+i+r(\TZ(x_{[n-k+1..n]})-i)\\
        ={}&{}g(1^{n-i}0^i)+(r-1)(\TZ(x_{[n-k+1..n]})-i)
        \le g(1^{n-i}0^i),
    \end{align*}
    where the first inequality uses $\LO(x_{[n-k+1..n]})+\TZ(x_{[n-k+1..n]})\le k$, and the second inequality uses $r\ge 1$ and $\TZ(x_{[n-k+1..n]})<i$ discussed above. 
    This shows that $x= 1^{n-i}0^i$ is a maximizer of (\ref{eq:epflo}). The inequalities above are equalities only for $r=1$ and $\LO(x_{[n-k+1..n]})+\TZ(x_{[n-k+1..n]})= k$, which completes our proof.
\end{proof}
From Theorem~\ref{thm:optlo}, we obtain the following corollary to show that the possible difficulty of using the nonparameter penalty way (which is identical to the exterior penalty way with $r=1$ as discussed before) to maximize (\ref{eq:epslo}). 

\begin{corollary}
    Let $M$ be defined as in Lemma~\ref{lem:PFlo} and $k>2$. For $i\in[0..k]$, let $\eps_i=n-k+i$, and $x_i$ be one maximizer of the nonparameter penalty function for the maximization of (\ref{eq:epslo}) with $\eps=\eps_i$.
    Let $F=\{h^k(x_{i})\mid i\in [0..k]\}$. Then  {$M \subseteq F$} does not always hold.
    \label{cor:nplo}
\end{corollary}

Different from the difficulty of the nonparameter penalty function as in Corollary~\ref{cor:nplo}, from Theorems~\ref{lem:PFlo} and~\ref{thm:optlo}, we know that the exterior penalty way with $r>1$ to solve (\ref{eq:epflo}) with $\eps\in\{\eps_i\mid i\in[0..k]\}$ will cover the full Pareto front of $\lotz_k$. To obtain an overall performance for the further comparison with the MOEAs, now we consider the runtime of the randomized local search algorithm to maximize (\ref{eq:epflo}) with $\eps=\eps_i$ (for any given $i\in[0..k]$) and $r>1$. The following two lemmas show the characteristics of the survivals from an infeasible solution and from a feasible solution respectively.
\begin{lemma}
Let $S^*$ be defined as in Lemma~\ref{lem:PFlo}. Let $i\in[0..k]$ and $\eps_i=n-k+i$. Consider using the randomized local search algorithm to maximize $g$ defined as in (\ref{eq:epflo}) with $\eps=\eps_i$ and $r>1$. Let $x$ be an infeasible solution, and $y$ be the offspring generated by flipping the $j$-th ($j\in[1..n]$) bit of $x$. 
\begin{itemize}
    \item[(1)] If $x\notin S^*$, $y$ survives iff $j\in [\LO(x)+1..n-\TZ(x_{[n-k+1..n]})]$. We have
$g(y)>g(x)$ iff $j=\LO(x)+1$ or both $j=n-\TZ(x_{[n-k+1..n]})$ and $\TZ(x_{[n-k+1..n]})<k$ hold; 
    \item[(2)] If $x\in S^*$, $y$ survives iff $g(y)>g(x)$, that is, iff $j=\LO(x)$. 
\end{itemize}
\label{lem:infealo}
\end{lemma}
\begin{proof}
    We first discuss the case that this infeasible $x$ is not a Pareto optimum. In this case, we know that $x$ has the form of $1^{\LO(x)}0*10^{\TZ(x_{[n-k+1..n]})}$ with $*\in\{0,1\}^{n-2-\LO(x)-\TZ(x_{[n-k+1..n]})}$ if  {$\TZ(x_{[n-k+1..n]})<k$, or the form of $1^{\LO(x)}0*0^k$ with $*\in\{0,1\}^{n-k-1-\LO(x)}$ if $\TZ(x_{[n-k+1..n]})=k$.}  {In the second 
case, if in addition $x= 1^{n-k-1}0^{k+1}$, then we have $\LO(x)+1=n-\TZ(x_{[n-k+1..n]})$; in all other 
situations, we have $\LO(x)+1<n-\TZ(x_{[n-k+1..n]})$.}
    
    
    Since $x$ is infeasible, we have $g(x)=h_1^k(x)+r(h_2^k(x)-\eps)$. 
    As $x$ is not Pareto optimal, based on the form of $x$ analyzed above, we know that a bit flip in the first $\LO(x)$ positions gives $h_2^k(y)\le h_2^k(x)\le \eps_i$ and $h_1^k(y)<h_1^k(x)$, and we know that
 a bit flip in the last $\TZ(x_{[n-k+1..n]})$ positions gives $h_1^k(y)=h_1^k(x)$ and $h_2^k(y)< h_2^k(x)\le \eps_i$. Consequently, we have $g(y)<g(x)$ and $y$ will not survive.

    Now we consider the flip of the $j$-th bit for $j\in [\LO(x)+1..n-\TZ(x_{[n-k+1..n]})]$.  {Recall the already proved fact that $\LO(x)+1< n-\TZ(x_{[n-k+1..n]})$ (except that $\LO(x)+1= n-\TZ(x_{[n-k+1..n]})$ if $x=1^{n-k-1}0^{k+1}$)}. 
    \begin{itemize}
        \item If $j=\LO(x)+1$  {(including $j=\LO(x)+1=n-\TZ(x_{[n-k+1..n]})$ if $x=1^{n-k-1}0^{k+1}$)}, 
    then $h_1^k(y)\ge h_1^k(x)+1$ and $h_2^k(y)\ge h_2^k(x)$. Thus $g(y)\ge h_1^k(y)+r(h_2^k(y)-\eps) > g(x)$, and $y$ will survive. 
        \item For $j\in [\LO(x)+2..n-\TZ(x_{[n-k+1..n]}) - 1]$ (if this exists), we have $h^k(x)=h^k(y)$, and $y$ will survive since $g(x)=g(y)$;
        \item Let $j=n-\TZ(x_{[n-k+1..n]})$  {(excluding the case when $x=1^{n-k-1}0^{k+1}$)}. When $\TZ(x_{[n-k+1..n]})=k$, since $\LO(x)+1 < n-\TZ(x_{[n-k+1..n]})=n-k$, we know $h^k(x)=h^k(y)$, and $y$ will survive. When $\TZ(x_{[n-k+1..n]})<k$, 
         {we have} $h_1^k(y)=h_1^k(x)$ and $h_2^k(y)= h_2^k(x)+1$, and  {then} $g(y)\ge h_1^k(y)+r(h_2^k(y)-\eps) > g(x)$. Thus $y$ will survive.
    \end{itemize}  

    Now it remains to discuss the case when 
    $x$ is a Pareto optimum. Then $\LO(x)+\TZ(x_{[n-k+1..n]})=n$ and thus $\LO(x)\ge n-k$. Since $\LO(x)=n-k$ (that is, $x=1^{n-k}0^k$) makes $h_2^k(x)=n\ge\eps_i$ for any $i\in[0..k]$, we know that $x=1^{n-k}0^k$ is feasible. Hence, in the following, we only consider $\LO(x)>n-k$.
    When $j\in[1..n-k]$, we have $h^k(y)<h^k(x)$. Thus $y$ is infeasible and $g(y)<g(x)$. When $j\in [n-k..\LO(x)-1]$, we have $h_1^k(y)<h_1^k(x)$ and $h_2^k(y)=h_2^k(x)$. Thus $y$ is infeasible and $g(y)<g(x)$. When $j=\LO(x)$, we have $h_1^k(y)=h_1^k(x)-1$ and $h_2^k(y)=h_2^k(x)+1$. Then
    $$g(y)\ge h_1^k(x)-1+r(h_2^k(x)+1-\eps_i)=g(x)+r-1>g(x),$$ 
    where the last inequality uses $r>1$. Thus $y$ will survive. When $j\in[\LO(x)+1..n]$ (if exists), we have $h_1^k(y)\le h_1^k(x)+1$ and $h_2^k(y)<h_2^k(x)<\eps_i$, where the last inequality uses the fact that $x$ is infeasible. Hence, $y$ is infeasible and
    \[
    g(y)=h_1^k(y)+r(h_2^k(y)-\eps_i)\le h_1^k(x)+1+r(h_2^k(x)-1-\eps_i)=g(x)+1-r<g(x),
    \]
    where the last inequality uses $r<1$. 
\end{proof}

\begin{lemma}
Let $i\in[0..k]$ and $\eps_i=n-k+i$. Consider using the randomized local search algorithm to maximize $g$ defined as in (\ref{eq:epflo}) with $\eps=\eps_i$ and $r>1$. Let $x$ be a feasible solution and $y$ be the offspring generated by flipping the $j$-th ($j\in[1..n]$) bit of $x$. 
\begin{itemize}
    \item[I.] $x$ is the unique optimum iff $\LO(x)=n-i$, and no $y$ can survive once the optimum $x$ is reached; 
    \item[II.] Conditional on that $x$ is not the optimum, $y$ survives iff $(j\in [\LO(x)+1..\LO(x)+k-i],\LO(x)\le n-k)$, or $(j\in [\LO(x)+1..n-i], \LO(x)\in [n-k..n-i-1])$; 
    \item[III.] Conditional on that $x$ is not the optimum, $g(y)>g(x)$ iff $j=\LO(x)+1$; 
    \item[IV.] {If $y$ survives, $y$ is} feasible as well. 
\end{itemize}
\label{lem:fealo}
\end{lemma}
\begin{proof}
 {We first prove the fact I.} For a feasible $x$ with $\LO(x)\ge n-k$, we know that $n-k+i\le \LO(x_{[1..n-k]})+\TZ(x_{[n-k+1..n]})=n-k+\TZ(x_{[n-k+1..n]})$. Hence $\TZ(x_{[n-k+1..n]})\ge i$, that is, $x_{[n-i+1..n]}=0^i$. Thus $\LO(x_{[n-k+1..n]})\le k-i$ and $\LO(x)\le n-k+k-i=n-i$. 
If further $\LO(x)=n-i$, then with the above $x_{[n-i+1..n]}=0^i$, we know that $x=1^{n-i}0^i$, which is the unique optimum (for given $\eps=\eps_i$ and $r>1$) from Theorem~\ref{thm:optlo}. As $x=1^{n-i}0^i$ results in $\LO(x)=n-i$, the first part of the  {fact I} is proved. Since $x=1^{n-i}0^i$ is the unique optimum as discussed  { and $y\ne x$ after the one-bit mutation}, we know that no $y$ will have $g(y)\ge g(x)$ and that no replacement will happen. Then the  {fact I} is proved. Noting that $1^{n-k}0^k$ is the unique optimum and the only feasible solution for $i=k$, we only consider $i\in[0..k-1]$ in the following.

 {We then prove the facts II to IV for a} feasible $x$  {with} $\LO(x)\in [n-k..n-i-1]$. Then $x_{[n-i+1..n]}=0^i$ as discussed above. Since $j\in[1..\LO(x)]$ gives $h_1^k(y)<h_1^k(x)$, we have $g(y) \le h_1^k(y) < h_1^k(x) = g(x)$, and thus $y$ will not survive. When
$j=\LO(x)+1$, we have $h_1^k(y)\ge h_1^k(x)+1\ge n-k+1$ and $y_{[n-i+1..n]}=x_{[n-i+1..n]}=0^i$. Thus $h_2^k(y)\ge n-k+i$. Hence, $y$ is feasible and $g(y)>g(x)$. When $j\in[\LO(x)+2..n-i]$ (if exists), $h_1^k(y)=h_1^k(x)$ and $h_2^k(y) \ge h_2^k(x) \ge n-k+i$. Thus $y$ is feasible and $g(y)=g(x)$. When $j\in[n-i+1..n]$, we know that $h_1^k(y)=h_1^k(x)$ and $h_2^k(y)\le n-k+i-1$. Hence, $g(y)=g(x)+r(h_2^k(y)-(n-k+i))<g(x)$, and $y$ will not survive.  {Thus, the facts III and IV and the last part of the fact II are proven for the case when the feasible $x$ has $\LO(x)\in [n-k..n-i-1]$.}

It remains to  {prove the facts II to IV} when the feasible $x$ has $\LO(x)< n-k$. Then $\LO(x)=\LO(x_{[1..n-k]})$. Since $x$ is feasible, we know that $\LO(x)+\TZ(x_{[n-k+1..n]})\ge n-k+i$. Thus $\LO(x)+k-i \ge n-\TZ(x_{[n-k+1..n]})$ and then \begin{align}
x_{[\LO(x)+k-i+1..n]}=0^{n-\LO(x)-k+i}.
\label{eq:xnk}
\end{align}
When $j\in [1..\LO(x)]$, we have $h_1^k(y) < h_1^k(x)$. Thus $g(y)\le h_1^k(y)<h_1^k(x)=g(x)$ and $y$ will not survive. When $j=\LO(x)+1$, which is also the case of $j=\LO(x)+k-i$ for $i=k-1$, then $h_1^k(y) \ge h_1^k(x) +1$ and $h_2^k(y) \ge h_2^k(x) +1 > n-k+i$, hence $g(y)>g(x)$. When $j\in[\LO(x)+2..\LO(x)+k-i]$ for $i\in[0..k-2]$, we know that $h_1^k(y)=h_1^k(x)$ and $$
h_2^k(y)=\LO(y)+\TZ(y_{[n-k+1..n]}) \ge \LO(x)+n-\LO(x)-k+i=n-k+i,
$$
where the inequality uses (\ref{eq:xnk}). Hence, $y$ is feasible and $g(y)=g(x)$. When $j\in [\LO(x)+k-i+1..n]$ for any $i\in[0..k-1]$, with (\ref{eq:xnk}) we know that $h_1^k(y)=h_1^k(x)$ and 
\begin{align*}
    h_2^k(y)={}&{} \TZ(y_{[n-k+1..n]})+\LO(y_{[1..n-k]}) \\
    \le{}&{} (n-(\LO(x_{[1..n-k]})+k-i+2)+1)+\LO(x_{[1..n-k]})\\
    ={}&{}n-k+i-2+1=n-k+i-1<\eps_i.
\end{align*}
Thus $y$ is infeasible and
\begin{align*}
    g(y){}&{}= h_1^k(y)+r(h_2^k(y)-\eps_i) =h_1^k(x)+r(h_2^k(y)-\eps_i) < h_1^k(x)= g(x).
\end{align*}
 {Thus, the facts III and IV and the first part of the fact II are proven for the case when the feasible $x$ has $\LO(x)<n-k$.}

 {Noting that $\LO(x)+k-i=n-k+k-i=n-i$ for $\LO(x)=n-k$ w.r.t. the fact II,} this lemma is proved.
\end{proof}

Now we have the following runtime result.
\begin{theorem}
Let $i\in[0..k]$ and $\eps_i=n-k+i$. Consider using the randomized local search algorithm to maximize $g$ defined as in (\ref{eq:epflo}) with $\eps=\eps_i$ and $r>1$. Then after $O(n^2)$ iterations in expectation, the optimal solution $1^{n-i}0^i$ is reached.
\label{thm:rls}
\end{theorem}
\begin{proof}
    We pessimistically assume that $x$ starts from an infeasible solution. From Lemma~\ref{lem:infealo}, we know that the flip  {in} any position that makes the offspring survive will not decrease the value of $h_2^k$.  {We consider the flip in the $j$-th position, where 
    \begin{align*}
    \begin{cases}
        j=n-\TZ(x_{[n-k+1..n]}), & \text{ if } x\notin S^* \text{ and }\TZ(x_{[n-k+1..n]})< k,\\
        j= \LO(x)+1, & \text{ if } x\notin S^* \text{ and }\TZ(x_{[n-k+1..n]}) = k,\\
        j=\LO(x), & \text{ if } x\in S^*.
    \end{cases}
    \end{align*}
    Also from Lemma~\ref{lem:infealo}, we know that such a flip} will strictly increase the value of $g$, and also strictly increase the value of $h_2^k$. 
    Since the probability to flip a specific bit is $1/n$, we know that an improvement of $h_2^k$ will be witnessed in  {at most} $n$ expected number of iterations (function evaluations). Hence, to reach a feasible solution, that is, a solution with its $h_2^k$ value at least $\eps_i=n-k+i$, we need at most $n(n-k+i)$ expected number of function evaluations.  

    Once a feasible solution is reached, from Lemma~\ref{lem:fealo} we know that all individuals in future iterations will be feasible  {(that is,} $g=h_1^k${), and $h_1^k(=g)$} does not decrease. Also from Lemma~\ref{lem:fealo}, only the flip  {in the} $(\LO(x)+1)$-th bit for the current individual $x$ will result in the improvement of $h_1^k(=g)$. 
    Hence, with the probability of $1/n$ to flip the $(\LO(x)+1)$-th bit, we know that in expected $n$ iterations (function evaluations), an improvement of $h_1^k$ will be witnessed. As the unique optimum is reached when $h_1^k=n-i$ (from Lemma~\ref{lem:fealo} (1)), we know that the optimum is reached in  {at most} $n(n-i)$  {expected} number of function evaluations.

    Hence,  {this theorem is proved}.
\end{proof}

\subsection{Efficiency of the MOEAs}
For this brief discussion, we only consider the non-decomposition MOEAs ((G)SEMO, \NSGA, and \SMS) as examples to see how the MOEAs maximize $\lotz_k$. The following theorem gives the runtime result.

\begin{theorem}
Consider using (G)SEMO, SMS-EMOA with $\mu \ge k+1$, or the NSGA-II with $N\ge 4(k+1)$ to optimize the $\lotz_k$ problem. Then the expected number of function evaluations to cover the full Pareto front is $O(\max\{k,1\}n^2)$ for (G)SEMO, $O(\mu n^2)$ for the SMS-EMOA, and $O(Nn^2)$ for the NSGA-II.
\label{thm:moeaslo}
\end{theorem}
\begin{proof}[Proof sketch]
    The proof idea is similar to the one for the $\omm_k$ problem. Here, we consider two phases of the optimization process. The first phase starts from the initial generation and ends when $1^n$ is reached for the first time. The second phase starts right after the end of the first phase and ends when the entire Pareto front is covered. 
    
    For the first phase, for some iteration $t$ let $h_{1,\max}^{k}=\max_{z\in P_t}\{h_1^k(z)\}$ be the maximal $h_1^k$ value in the population $P_t$.  {Note that Lemma~\ref{lem:sur} holds} for all problems where the maximal number of mutually incomparable solutions is $k+1$, thus also holds for $\lotz_k$. Hence, $h_{1,\max}^{k}$ will not decrease, and will reach its maximum of $n$ iff $1^n$ is reached. We pessimistically consider the following case when a strict improvement happens. In one iteration, one individual $x$ in  {$A_t:=\{z\mid h_1^k(z)=h_{1,\max}^{k}\}$}  is selected as a parent (for the \NSGA we further require $\cDis(x)=+\infty$) and generates an offspring $y$ with $h_1^k(y)> h_{1,\max}^{k}$ (which is easily reached when only the $(h_{1,\max}^{k}+1)$-th bit of $x$ is flipped). Then  {from Lemma~\ref{lem:sur}, we know that in the next generation there is a $z$ such that $z\succeq y$, thus $h_1^k(z)\ge h_1^k(y)> h_{1,\max}^{k}$. Also from}
    Lemma~\ref{lem:sur}, we know that $A_s$ for $s>t$ will not be empty before an improvement of $h_{1,\max}^{k}$ is witnessed. 
    Noting that the probability of generating such a $y$ when $x$ is chosen for mutation is $\Omega(1/n)$, together with the probability of choosing $x$ as a parent (discussed in Lemmas~\ref{lem:parent} and~\ref{lem:firstP}), we know that the expected number of function evaluations to see an improvement of $h_{1,\max}^{k}$ is $O(kn)$ for the GSEMO, $O(\mu n)$ for the \SMS, and $O(N n)$ for the \NSGA. Hence, to reach $h_{1,\max}^{k}=n$, that is, to reach $1^n$ as discussed before, the expected number of function evaluations that we need is $O(kn^2)$ for the GSEMO, $O(\mu n^2)$ for the \SMS, and $O(N n^2)$ for the \NSGA. 

    For the second phase, we pessimistically consider of sequentially generating $1^{n-1}0,\dots,1^{n-k}0^k$. Similar to the above discussion, we only need to focus on selecting $1^{i+1}0^{n-i-1}$ as a parent to generate $1^{i}0^{n-i}$ with $i\in[n-k..n-1]$. As there are at most $k$ remaining missing Pareto front points when $1^n$ is reached for the first time, we know that the additional expected number of evaluations to cover the full Pareto front is $O(k^2n)$ for the GSEMO, $O(\mu nk)$ for the \SMS, and $O(N nk)$ for the \NSGA.
\end{proof}

\section{Conclusion}\label{sec:con}
With the benchmark class $\omm_k$ to depict different degrees of conflict between two objectives, we proved that the scalarization approach has difficulty in covering the full Pareto front, and that $\eps$-constraint approach covers the full Pareto front in expected $O(\max\{k,1\}n\ln n)$ function evaluations but the solving method and the parameter settings need to be carefully chosen. 

In contrast {to a good distribution of the set of parameter $\eps$ and the property penalty coefficient $r$ for the exterior penalty function way to successfully solve the corresponding $\eps$-constraint problems}, the MOEAs (including the basic (G)SEMO, and the modern \NSGA, \SMS, and MOEA-D) can cover the full Pareto front in expected $O(\max\{k,1\}n\ln n)$ function evaluations {with less effort. The \NSGA and \SMS (the MOEA/D) require, instead of a good distribution,} the population size {(the number of subproblems)} at least a constant factor larger than the maximal size of incomparable set. This showed that the MOEAs easily achieve at least the same asymptotic runtime as the  {non-MOEA approaches}, but without careful reformulation of the bi-objective problem. This also showed that even in the extreme case where two objectives are the same, the MOEAs have the same asymptotic runtime as solving the corresponding single-objective problem. 
{We also saw the similar findings on a bi-objective \lo variant.} It is the first work that theoretically compares the performance of the MOEAs and non-MOEA methods.   

{This paper demonstrated that the reasonable runtime of the MOEAs without decomposition is essentially guaranteed by two properties: (1) the ability to retain all reached Pareto front points and (2) an easy way to generate new Pareto front points from the already reached ones. For bi-objective optimization, both properties are easily satisfied: the dominance-based criterion in the survival selection, and the standard variation operator ensuring the reachability of all solutions. The reasonable runtime for the MOEA/D stems from a high-quality decomposition and a good runtime for each resulting single-objective subproblem. It is an interesting future research question about the comparison on the approximation guarantees when the reached Pareto front needs to be discarded for the MOEAs without decomposition, or when the decomposition employed by MOEA/D poorly aligns with the Pareto front. Besides, the additional essential reason, in our discussions, for these MOEAs with the same runtime as the non-MOEAs with the deliberate designs on the parameters, lies in the fact that the maximal size of mutually incomparable solutions equals the degree of conflict. Whether this fact holds or not depends on the problems to solve. Hence, it is also an interesting research topic to discuss other types of problems to represent different degrees of conflict, and the performance comparison among the MOEAs and the non-MOEAs.}


\section*{Acknowledgments}
This work was supported by National Natural Science Foundation of China (Grant No. 62306086) and Guangdong Basic and Applied Basic Research Foundation (Grant No. 2025A1515011936).


\newcommand{\etalchar}[1]{$^{#1}$}

\end{document}